\PassOptionsToPackage{table}{xcolor}

\documentclass[sigconf,nonacm]{acmart}

\AtBeginDocument{%
}

\renewcommand\footnotetextcopyrightpermission[1]{}
\usepackage{amsmath}
\usepackage{amsfonts}
\usepackage{amsthm}
\usepackage{dirtytalk}
\usepackage{booktabs}
\usepackage[normalem]{ulem}
\usepackage{tabularx}
\usepackage{enumitem}
\usepackage{xspace}

\usepackage{algorithm}
\usepackage{algpseudocode}

\usepackage{graphicx}
\usepackage{pifont}

\graphicspath{{Figure/}{figures/}{./}}

\newtheorem{theorem}{Theorem}
\newtheorem{assumption}{Assumption}
\newtheorem{proposition}{Proposition}
\newtheorem{lemma}{Lemma}
\newtheorem{corollary}[theorem]{Corollary}

\newcommand{\methodname}{\textsc{HP-JEPA}\xspace}

\newcommand{\cmark}{\ding{51}}
\newcommand{\xmark}{\ding{55}}
\newcommand{\pmark}{\ensuremath{\triangle}}

\algrenewcommand\algorithmicrequire{\textbf{Input:}}
\algrenewcommand\algorithmicensure{\textbf{Output:}}

\newcommand{\safeincludegraphics}[3][]{%
  \IfFileExists{#2}{%
    \includegraphics[#1]{#2}%
  }{%
    \fbox{%
      \begin{minipage}{#3}
        \centering
        \vspace{0.25cm}

        Missing figure file:
        \texttt{\detokenize{#2}}.

        \vspace{0.25cm}
      \end{minipage}%
    }%
  }%
}

\newcommand{\lin}[1]{}
\newcommand{\RX}[1]{}
\newcommand{\wg}[1]{}
\newcommand{\JX}[1]{}

\begin{document}

\title{%
  HP-JEPA: Hierarchical Partitioning for Multi-Resolution Graph
  Joint-Embedding Predictive Learning%
}


\author{Ruichen Xu}
\affiliation{%
  \department{Department of Applied Mathematics and Statistics}
  \institution{Stony Brook University}
  \city{Stony Brook}
  \state{NY}
  \country{USA}
}

\author{Jingxiang Qu}
\affiliation{%
  \department{Department of Computer Science}
  \institution{Stony Brook University}
  \city{Stony Brook}
  \state{NY}
  \country{USA}
}

\author{Wenhan Gao}
\affiliation{%
  \department{Department of Applied Mathematics and Statistics}
  \institution{Stony Brook University}
  \city{Stony Brook}
  \state{NY}
  \country{USA}
}

\author{Jiaxing Zhang}
\authornote{Co-corresponding authors: Jiaxing Zhang
(\texttt{tabzhangjx@gmail.com}) and Yuefan Deng
(\texttt{yuefan.deng@stonybrook.edu}).}
\affiliation{%
  \institution{Independent Research}
  \city{Newark}
  \state{NJ}
  \country{USA}
}

\author{Linsey Pang}
\affiliation{%
  \institution{PayPal}
  \city{San Jose}
  \state{CA}
  \country{USA}
}

\author{Ravid Shwartz-Ziv}
\affiliation{%
  \department{Center for Data Science}
  \institution{New York University}
  \city{New York}
  \state{NY}
  \country{USA}
}

\author{Yann LeCun}
\affiliation{%
  \department{Center for Data Science}
  \institution{New York University}
  \city{New York}
  \state{NY}
  \country{USA}
}

\author{Yuefan Deng}
\authornotemark[1]
\affiliation{%
  \department{Department of Applied Mathematics and Statistics}
  \institution{Stony Brook University}
  \city{Stony Brook}
  \state{NY}
  \country{USA}
}

\renewcommand{\shortauthors}{Xu et al.}
\begin{abstract}
Graph self-supervised learning aims to learn transferable
representations from large-scale unlabeled graph data.
Joint-embedding predictive architectures (JEPAs) avoid explicit
negative-pair construction and raw-input reconstruction by predicting
masked targets directly in latent space. However, existing graph JEPAs
typically rely on a single predefined graph partition, biasing the
learned representations toward one structural granularity and limiting
their ability to capture complementary patterns at different graph
scales. To address this limitation, we propose HP-JEPA, a hierarchical
partitioning framework for multi-resolution graph joint-embedding
prediction. HP-JEPA organizes each graph into an ordered bank of
coarse-to-fine partition resolutions and performs context--target
latent prediction separately at each resolution using an online
encoder, an exponential-moving-average target encoder, and a latent
predictor. The resulting resolution-specific graph representations are
subsequently integrated through concatenation or task-specific
resolution weighting, allowing downstream models to combine
complementary local, regional, and global structural information.
Experiments on seven graph classification benchmarks and one graph
regression benchmark show that HP-JEPA outperforms the fixed-resolution
Graph-JEPA baseline on 6 of 8 tasks, improving upon Graph-JEPA on most
evaluated benchmarks. Size-stratified analyses further show that
HP-JEPA achieves higher accuracy than Graph-JEPA in most evaluated
graph-size quartiles on three representative datasets. These results
highlight the effectiveness of hierarchical multi-resolution
partitioning for transferable graph representation learning.
\end{abstract}
\maketitle

\vspace{-5pt}

\section{Introduction}
\label{sec:introduction}
\begin{figure*}[t]
    \centering
    \includegraphics[
        width=\textwidth,page = 1, 
    ]{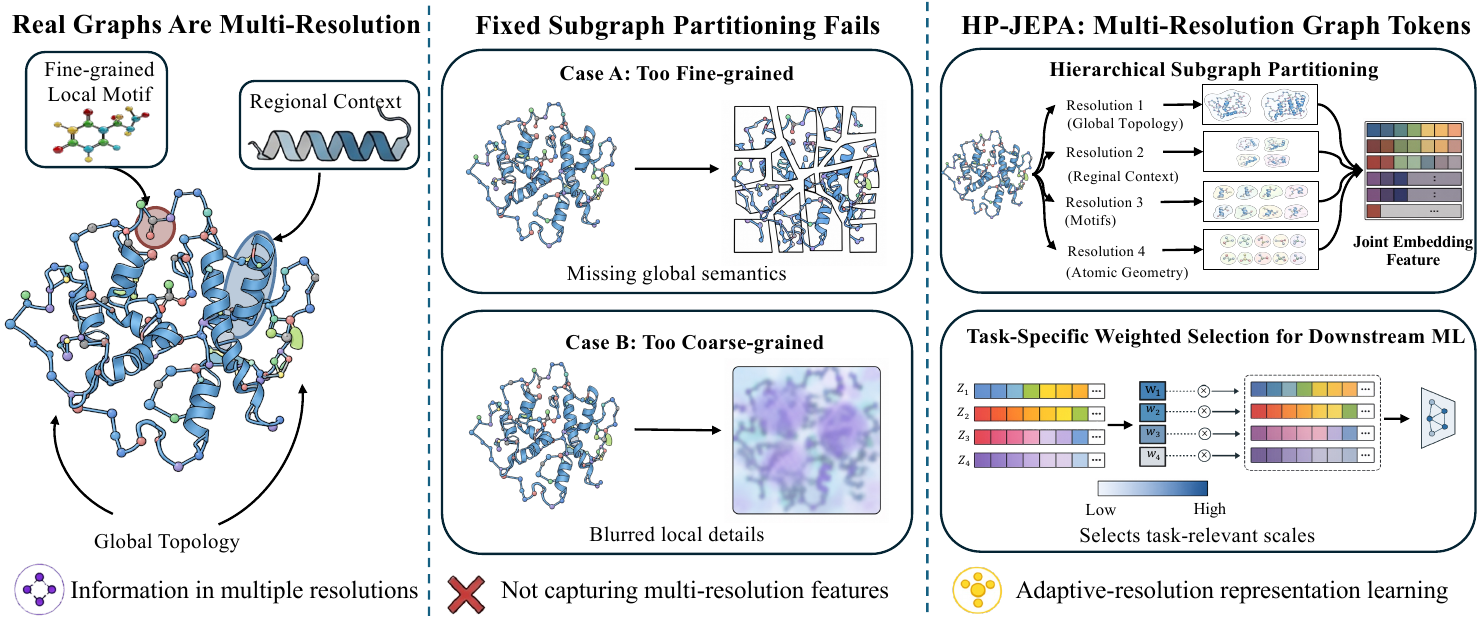}
    \vspace{-14pt}
    \caption{\textbf{Motivation and overview of HP-JEPA.}
    Real-world graphs contain complementary structural information at
    multiple resolutions, including fine-resolution local motifs,
    intermediate-resolution regional context, and coarse-resolution
    global topology.
    A single fixed resolution can be overly fine, fragmenting global
    context, or overly coarse, obscuring local structure.
    HP-JEPA constructs graph-region tokens at multiple resolutions and
    integrates their resolution-specific representations.
    Task-specific resolution weighting allows downstream models to
    emphasize the most informative resolutions. 
    }
    \label{fig:overview}
    \vspace{-10pt}
\end{figure*}

Graph representation learning aims to encode graph-structured data
into transferable representations that support downstream
classification and regression tasks~\cite{hamilton2017representation,wu2021comprehensive, xu2019powerful}.
Because task-specific annotations are often limited or expensive to
obtain, self-supervised learning has become an important paradigm for
learning from large collections of unlabeled graphs.
Existing graph self-supervised methods primarily construct their
training signals through contrastive learning
\cite{sun2019infograph,you2020graphcl,hassani2020contrastive,suresh2021adversarial},
generative architectures
\cite{hou2022graphmae,tan2023s2gae},
self-predictive~\cite{thakoor2021large,xie2022lagraph} or joint-embedding predictive architectures (JEPAs)
\cite{assran2023ijepa,skenderi2025graphjepa}.
Contrastive methods depend heavily on graph augmentations and, in many
cases, negative-pair construction~\cite{velickovic2019deep, qiu2020gcc, zhang2021canonical, huang2025does}, while generative architectures optimize
the reconstruction of raw node attributes or graph structures. Self-predictive methods try to predict the unobserved components of the graph, and supervise this process on original data space. JEPAs instead predict
held-out targets directly in representation space, avoiding both
explicit negative pairs and raw-input reconstruction~\cite{bardes2023v, lei2025mjepa}.
This formulation offers a promising route toward learning high-level structural and semantic dependencies in graphs.

Graph-JEPA~\cite{skenderi2025graphjepa} extends joint-embedding prediction to graph-level
representation learning by partitioning each graph into a collection
of regions or patches.
During pretraining, visible context regions are encoded by an online
encoder, while held-out target regions are represented by an
exponential-moving-average target encoder.
A predictor then estimates the latent target representations from the
visible context and target-specific structural queries.
This formulation successfully transfers JEPA-style learning to
graph-structured data without reconstructing raw node or edge
information.
However, the graph is tokenized using a single predefined partition
resolution throughout pretraining.
Consequently, the predictive objective is restricted to dependencies
expressed at one structural granularity, and the resulting
representation may become specialized to the particular resolution induced
by that partition.


The fixed-resolution assumption is restrictive because graph
semantics are inherently multi-granular~\cite{coarse2fine, Grain, dontwalk}.
Fine-grained partitions preserve localized motifs and short-range
interactions, but may fragment regional organization and long-range
structural dependencies.
Conversely, coarse partitions capture broader graph topology and
regional context, but can blur discriminative local patterns.
The most informative granularity may also vary across graphs of
different sizes and across downstream tasks: a resolution suitable for
identifying local biochemical motifs, for example, may not be optimal
for characterizing global interaction patterns in social graphs.
Therefore, no single partition resolution can be expected to preserve
all task-relevant information or generalize uniformly across
heterogeneous graph resolutions.
These observations motivate graph representation learning over an
ordered hierarchy of partition resolutions rather than a single fixed
partition, as shown in Fig.~\ref{fig:overview}.

To address this limitation, we propose
\textbf{HP-JEPA}, a hierarchical partitioning framework for
multi-resolution graph joint-embedding prediction.
HP-JEPA represents each graph using an ordered bank of coarse-to-fine
partition resolutions.
Here, \emph{hierarchical} refers to the ordering of structural
granularities: partitions at different resolutions are constructed
from the full graph.
At each active resolution, graph regions are encoded as patch-content
tokens together with structural queries that describe their positions
and contexts in the original graph.
HP-JEPA then performs resolution-wise context--target prediction using
a shared online encoder, an exponential-moving-average target encoder,
and a latent predictor.
By applying the predictive objective separately across resolutions,
HP-JEPA exposes the representation learner to structural dependencies
ranging from fine local patterns to coarse graph-level organization,
while retaining the JEPA advantages of label-free latent prediction
without negative pairs or raw graph reconstruction.


After self-supervised pretraining, HP-JEPA pools the valid region-token representations within each resolution to obtain a set of resolution-specific graph embeddings. These embeddings can be concatenated to preserve information from all configured resolutions or combined through a smoothed, task-specific weighting mechanism that learns the relative importance of each resolution while keeping the pretrained encoder frozen. This design separates label-free multi-resolution representation learning from lightweight task-specific resolution selection. The PROTEINS case study further shows that HP-JEPA produces a visually more organized embedding with clearer separation across graph-size quartiles and class labels. Together, these results indicate that learning and integrating representations across partition resolutions provides a more flexible alternative to fixed-resolution graph joint-embedding prediction.


Our main contributions are summarized as follows:
\begin{itemize}
    \item
    We identify the fixed-resolution bias in existing graph
    joint-embedding predictive architectures and formulate
    multi-resolution graph tokenization as a key requirement for
    learning representations across heterogeneous graph granularities.
    \item
    We propose HP-JEPA, which combines an ordered hierarchy
    of graph partition resolutions, resolution-wise latent
    context--target prediction, and multi-resolution readout with
    optional task-specific resolution weighting.
    \item
    We evaluate HP-JEPA through downstream graph classification and
regression, a size-stratified representation case study, and
sensitivity analyses over graph-size quartiles and resolution-bank
configurations.
    HP-JEPA improves over Graph-JEPA on six of eight tasks and maintains
its advantage in most evaluated graph-size quartiles and
resolution-bank configurations.
 
\end{itemize}

\section{Related Work}
\label{sec:related-work}

\begin{table*}[!t]
\centering
\caption{Design-level comparison of graph representation learning methods. The check marks (\cmark) and cross marks (\xmark) indicate that the property is explicitly satisfied/unsatisfied by design, and a triangle (\pmark) indicates a partial match.}
\vspace{-4pt}
\label{tab:related_work_comparison}
\scriptsize
\setlength{\tabcolsep}{3.5pt}
\renewcommand{\arraystretch}{1.12}
\resizebox{\textwidth}{!}{%
\begin{tabular}{@{}llcccccc@{}}
\toprule
Method
& Main training signal
& Label-Free
& No Neg.
& No Raw Recon.
& Latent Pred.
& Region Tokens
& Multi-Resolution Tokens \\
\midrule
F-GIN~\cite{xu2019powerful}
& Supervised labels
& \xmark
& \cmark
& \cmark
& \xmark
& \xmark
& \xmark \\

InfoGraph~\cite{sun2019infograph}
& Local--global mutual information
& \cmark
& \xmark
& \cmark
& \xmark
& \pmark
& \xmark \\

GraphCL~\cite{you2020graphcl}
& Augmented-view contrast
& \cmark
& \xmark
& \cmark
& \xmark
& \xmark
& \xmark \\

MVGRL~\cite{hassani2020contrastive}
& Multi-view contrast
& \cmark
& \xmark
& \cmark
& \xmark
& \xmark
& \xmark \\

AD-GCL-FIX/OPT~\cite{suresh2021adversarial}
& Adversarial graph contrast
& \cmark
& \xmark
& \cmark
& \xmark
& \xmark
& \xmark \\

GraphMAE~\cite{hou2022graphmae}
& Masked feature reconstruction
& \cmark
& \cmark
& \xmark
& \xmark
& \xmark
& \xmark \\

S2GAE~\cite{tan2023s2gae}
& Masked structure reconstruction
& \cmark
& \cmark
& \xmark
& \xmark
& \xmark
& \xmark \\

BGRL~\cite{thakoor2021large}
& Bootstrap latent prediction
& \cmark
& \cmark
& \cmark
& \cmark
& \xmark
& \xmark \\

LaGraph~\cite{xie2022lagraph}
& Latent graph prediction
& \cmark
& \cmark
& \cmark
& \cmark
& \pmark
& \xmark \\

Graph-JEPA~\cite{skenderi2025graphjepa}
& Single-resolution patch prediction
& \cmark
& \cmark
& \cmark
& \cmark
& \cmark
& \xmark \\

\textbf{\methodname}
& Multi-resolution patch prediction
& \cmark
& \cmark
& \cmark
& \cmark
& \cmark
& \cmark \\
\bottomrule
\end{tabular}%
}
\vspace{0pt}
\begin{flushleft}
\footnotesize
\textbf{Notes.}
(1) ``No Neg.'' means the objective does not require explicit negative pairs.
(2) ``No Raw Recon.'' means the method does not reconstruct raw node attributes, edges, or graph structure.
(3) ``Latent Pred.'' means the target is a representation-space target.
(4) ``Region Tokens'' means graph regions, subgraphs, or patches are explicit learning units.
(5) ``Multi-Resolution Tokens'' means that graph-region tokens are explicitly constructed at multiple partition resolutions.
(6) For F-GIN, check marks under ``No Neg.'' and ``No Raw Recon.'' only indicate that these mechanisms are absent; F-GIN is still supervised.
AD-GCL-FIX and AD-GCL-OPT are merged because they share the same method design.
\end{flushleft}
\vspace{-4pt}
\end{table*}

Graph neural networks provide the standard encoder backbone for graph representation learning by aggregating information over graph neighborhoods. Representative message-passing models include GCN, GraphSAGE, MPNN, and GIN~\cite{kipf2017gcn,hamilton2017inductive,gilmer2017neural,xu2019powerful}. Based on these encoders, contrastive and mutual-information-based graph self-supervised methods learn representations by comparing graph views or local--global summaries. InfoGraph maximizes mutual information between graph-level and substructure representations~\cite{sun2019infograph}, GraphCL contrasts augmented graph views~\cite{you2020graphcl}, MVGRL contrasts first-order and diffusion views~\cite{hassani2020contrastive}, and AD-GCL learns adversarial graph augmentations~\cite{suresh2021adversarial}. These methods improve graph-level representation learning, but their supervision is mainly defined by view agreement or augmentation design rather than latent prediction over graph-region structures. \methodname addresses this gap by learning from context--target prediction in graph-region latent space instead of relying on a fixed contrastive view construction.

Masked graph autoencoding provides another line of graph self-supervised learning. GraphMAE masks node attributes and reconstructs masked features~\cite{hou2022graphmae}, while S2GAE masks graph structure and reconstructs missing edges~\cite{tan2023s2gae}. These methods avoid explicit negative pairs, but their objectives are still tied to reconstructing raw graph observations. Non-contrastive methods move closer to representation-space learning: BGRL predicts target representations from another augmented view without negative samples~\cite{thakoor2021large}, and LaGraph formulates self-supervised learning as latent graph prediction~\cite{xie2022lagraph}. However, these methods do not explicitly organize graph-region targets across a hierarchy of structural granularities. \methodname addresses this limitation by predicting latent targets without reconstructing node attributes or edges, while using graph regions at multiple partition resolutions as the learning units.

Graph-JEPA~\cite{skenderi2025graphjepa} is the closest prior method to our work because it applies JEPA-style context--target prediction to graph patches. It partitions a graph into subgraphs and predicts latent target-patch representations from context-patch representations, avoiding raw reconstruction. However, Graph-JEPA exposes only a single patch-token resolution to the predictive objective. This fixed-resolution design can miss dependencies that appear across local, mesoscopic, and coarse graph structures. \methodname extends Graph-JEPA by constructing graph tokens over an ordered bank of partition resolutions, performing latent prediction independently within each resolution, and integrating the resulting resolution-specific representations during downstream readout. Here, the coarse-to-fine resolution hierarchy refers to the ordering of structural granularities rather than nested or parent--child partitions. The difference between \methodname and existing methods is summarized in Table~\ref{tab:related_work_comparison}.
\section{Preliminaries}
\label{sec:preliminaries}

This section establishes the notation and background used by HP-JEPA. We first formulate graph representation learning for an attributed graph $G=(V,E,\mathbf X)$ and then review a single-resolution graph JEPA formulation. In this formulation, the graph is decomposed into a collection of regions, one context region and a set of target regions are sampled, and the target latent states are predicted using an online token encoder $f_\theta$, an exponential-moving-average (EMA) target token encoder $f_{\bar\theta}$, and a predictor $q_\psi$. Section~\ref{sec:method} extends this formulation to $L=L_{\mathcal D}$ independently constructed graph resolutions, performs prediction separately within each resolution, maps each target-token state to a two-coordinate latent target, and combines the resulting resolution-specific graph representations $\{h_G^{(\ell)}\}_{\ell=1}^{L}$.

\subsection{Graph Representation Learning}
\label{sec:prelim-graph-representation}

Let $G=(V,E,\mathbf X)$ denote an attributed graph, where $V$ and $E$ are the node and edge sets, $\mathbf X\in\mathbb R^{n\times d_x}$ is the node-feature matrix, and $n=|V|$. Edge attributes, when available, are treated as part of the graph input but omitted from the notation for simplicity. We consider a graph dataset $\mathcal D=\{G_i\}_{i=1}^{N_{\mathcal D}}$ sampled from an underlying graph distribution $p_{\mathcal G}$.

Graph representation learning aims to map each graph $G$ to a fixed-dimensional representation $h_G\in\mathbb R^d$ that summarizes its structural and semantic information. A typical graph encoder $\mathcal E_\vartheta$ first produces node representations $\{r_v\in\mathbb R^d:v\in V\}$ and then applies a permutation-invariant readout according to $\{r_v\}_{v\in V}:=\mathcal E_\vartheta(G)$ and $h_G:=\operatorname{Readout}(\{r_v:v\in V\})\in\mathbb R^d$. The readout may be instantiated using mean pooling, sum pooling, or another permutation-invariant aggregation operator, ensuring that $h_G$ is unaffected by node ordering and has a fixed dimensionality across graphs of different sizes.

For a downstream graph-level task with label $y_G$, a task-specific predictor $\pi_\phi$ produces $\widehat y_G:=\pi_\phi(h_G)$. Given a labeled downstream dataset $\mathcal D_{\mathrm{task}}=\{(G_i,y_i)\}_{i=1}^{N_{\mathrm{task}}}$, the predictor is trained by minimizing $\mathcal L_{\mathrm{task}}:=N_{\mathrm{task}}^{-1}\sum_{i=1}^{N_{\mathrm{task}}}\ell_{\mathrm{task}}(\pi_\phi(h_{G_i}),y_i)$, where $\ell_{\mathrm{task}}$ denotes a classification or regression loss. In self-supervised graph representation learning, the representation encoder is pretrained without downstream labels, after which the learned graph representations are evaluated using a lightweight task-specific predictor with the pretrained encoder either frozen or fine-tuned.
\begin{figure*}[h]
\centering
\includegraphics[width=\textwidth,page = 2]{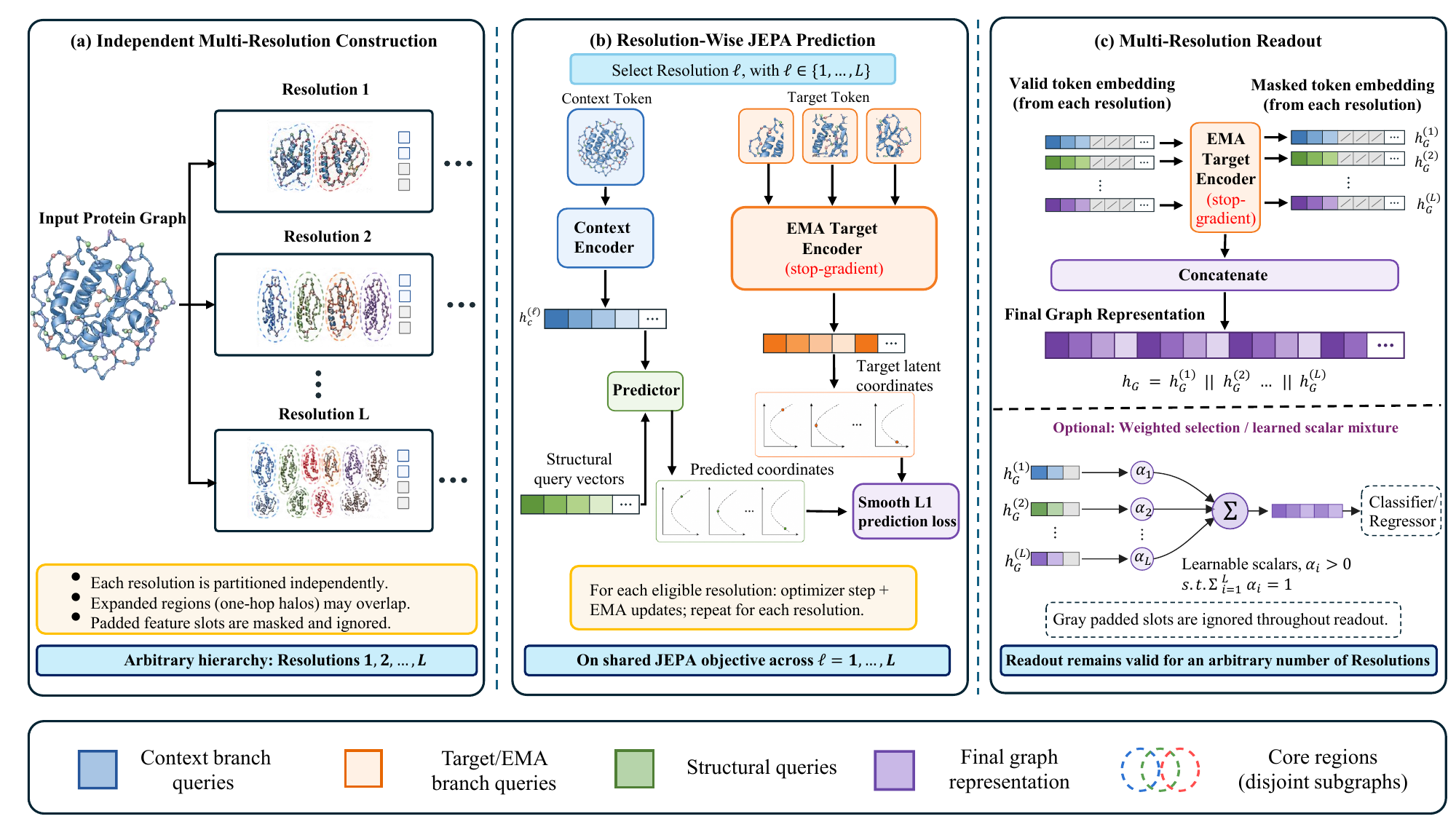}
\vspace{-16pt}
\caption{\textbf{Overview of HP-JEPA.} \textbf{(a)} The input graph is partitioned independently at $L$ configured resolutions. At resolution $\ell$, each core region $P_i^{(\ell)}$ defines a graph-patch support $\widetilde P_i^{(\ell)}$, while padded and inactive token slots are masked and ignored. \textbf{(b)} At each eligible resolution, the sampled context index $c_\ell(G)$ and target-index set $\mathcal T_\ell(G)$ determine the context and target-content inputs processed by the online token encoder $f_\theta$ and EMA target token encoder $f_{\bar\theta}$, respectively. For each $t\in\mathcal T_\ell(G)$, the predictor $q_\psi$ estimates $\widehat y_t^{(\ell)}$ from the context state $s_{c_\ell}^{(\ell)}$ and target structural query $p_t^{(\ell)}$, supervised by the two-coordinate target $y_{t,\mathrm{tgt}}^{(\ell)}$ using a coordinate-wise Smooth~$L_1$ loss. \textbf{(c)} During downstream readout, the frozen EMA target token encoder processes each active resolution independently, and its valid token representations $r_{G,i}^{(\ell)}$ are mean-pooled to obtain $h_G^{(\ell)}$. The resolution-specific representations are subsequently combined through concatenation $h_G^{\mathrm{concat}}$ or uniformly smoothed task-specific resolution weighting $h_G^{\mathrm{task}}$ for graph classification or regression.}
\label{fig:rh-jepa-overview}
\label{fig:method-overview}
\Description{HP-JEPA independently constructs graph regions at multiple resolutions, performs two-coordinate latent prediction separately within each eligible resolution, mean-pools valid token representations to obtain resolution-specific graph representations, and combines those representations through concatenation or task-specific resolution weighting.}
\vspace{-10pt}
\end{figure*}
\subsection{JEPA for Graphs}
\label{sec:prelim-graph-jepa}

Joint-embedding predictive architectures learn representations by predicting a target in latent space from a partially observed context rather than reconstructing the target in the original input space. Given a context input $x_{\mathrm{ctx}}$ and a target input $x_{\mathrm{tgt}}$, the online token encoder produces $s_{\mathrm{ctx}}:=f_\theta(x_{\mathrm{ctx}})\in\mathbb R^d$, while the EMA target token encoder produces $u_{\mathrm{tgt}}:=f_{\bar\theta}(\operatorname{sg}(x_{\mathrm{tgt}}))\in\mathbb R^d$, where $\operatorname{sg}(\cdot)$ denotes stop-gradient. A structural query $p_{\mathrm{tgt}}\in\mathbb R^d$ identifies the target whose representation should be predicted, and the predictor produces $\widehat u_{\mathrm{tgt}}:=q_\psi(s_{\mathrm{ctx}}+p_{\mathrm{tgt}})$. A generic JEPA objective is then written as $\mathcal L_{\mathrm{JEPA}}:=\ell_{\mathrm{pred}}(\widehat u_{\mathrm{tgt}},\operatorname{sg}(u_{\mathrm{tgt}}))$, where $\ell_{\mathrm{pred}}$ measures discrepancy in latent space. By defining supervision in the representation space, JEPA avoids explicit reconstruction of low-level input details and does not require negative samples.

To specialize this formulation to graph regions using the notation adopted in Section~\ref{sec:method}, consider a single predefined graph partition $\mathcal P(G):=\{P_i\}_{i=1}^{k(G)}$, where every $P_i$ is nonempty, $\bigcup_{i=1}^{k(G)}P_i=V$, and $P_i\cap P_j=\varnothing$ for $i\neq j$. Here, $k(G)$ denotes the number of valid regions in the partition. For each core region $P_i$, let $\widetilde P_i\supseteq P_i$ denote the support of its associated graph patch, and let $\rho_i$ denote a structural descriptor characterizing the position and structural context of that patch in the original graph.

A patch-token encoder $g_\eta$ maps the graph patch to a content token $z_i:=g_\eta(G[\widetilde P_i])\in\mathbb R^d$, while a structural-query encoder $e_\xi$ maps its descriptor to $p_i:=e_\xi(\rho_i)\in\mathbb R^d$. The context and downstream readout paths retain both patch content and structural information through $x_i^{\mathrm{ctx}}=x_i^{\mathrm{full}}:=z_i+p_i$, whereas the target path receives only the patch-content input $x_i^{\mathrm{tgt}}:=z_i$. Consequently, the target token encoder does not directly observe the target structural query supplied to the predictor.

Let $\mathcal I(G):=\{1,\ldots,k(G)\}$ denote the valid region-index set. A context index $c(G)$ is sampled from $\mathcal I(G)$, and a nonempty target-index set $\mathcal T(G)\subseteq\mathcal I(G)\setminus\{c(G)\}$ is sampled from the remaining regions; we write $M(G):=|\mathcal T(G)|$. Suppressing the graph argument when it is clear from context, the online token encoder produces $s_c:=f_\theta(x_c^{\mathrm{ctx}})$, while the EMA target token encoder produces $[u_t]_{t\in\mathcal T(G)}:=f_{\bar\theta}(\operatorname{sg}([x_t^{\mathrm{tgt}}]_{t\in\mathcal T(G)}))$. For each target index $t\in\mathcal T(G)$, the predictor estimates $\widehat u_t:=q_\psi(s_c+p_t)$, and the single-resolution graph JEPA objective is $\mathcal L_{\mathrm{Graph\text{-}JEPA}}(G):=M(G)^{-1}\sum_{t\in\mathcal T(G)}\ell_{\mathrm{pred}}(\widehat u_t,\operatorname{sg}(u_t))$.

Through this latent prediction objective, the single-resolution formulation learns statistical dependencies between a visible graph context and held-out target regions without reconstructing node attributes or edges. However, it represents the graph using only one predefined partition $\mathcal P(G)$ and therefore exposes the model to one structural granularity. HP-JEPA generalizes this formulation to $L$ independently constructed resolutions $\{\mathcal P^{(\ell)}(G)\}_{\ell=1}^{L}$ with configured capacities $K_\ell=2^\ell$ and graph-dependent valid-region counts $k_\ell(G)$. It replaces $c(G)$, $\mathcal T(G)$, and $M(G)$ with the resolution-specific quantities $c_\ell(G)$, $\mathcal T_\ell(G)$, and $M_\ell(G)$, and performs prediction independently within each eligible resolution.

Unlike the generic full-vector objective above, HP-JEPA derives a two-coordinate target from each EMA target-token state. Specifically, for $u_t^{(\ell)}\in\mathbb R^d$, it defines $m(u_t^{(\ell)}):=d^{-1}\mathbf 1_d^\top u_t^{(\ell)}$ and $y_{t,\mathrm{tgt}}^{(\ell)}:=\Phi(u_t^{(\ell)})=[\cosh(m(u_t^{(\ell)})),\sinh(m(u_t^{(\ell)}))]^\top\in\mathbb R^2$, while the predictor produces $\widehat y_t^{(\ell)}:=q_\psi(s_{c_\ell}^{(\ell)}+p_t^{(\ell)})$. After self-supervised pretraining, valid token representations are pooled separately within each resolution to obtain $\{h_G^{(\ell)}\}_{\ell=1}^{L}$, which are combined through concatenation or task-specific resolution weighting as described in Section~\ref{sec:method}.


\section{Method}
\label{sec:method}

We consider an attributed graph $G=(V,E,\mathbf X)$, where $V$ and $E$ are the node and edge sets, $\mathbf X\in\mathbb R^{n\times d_x}$ is the node-feature matrix, and $n=|V|$. Let $L:=L_{\mathcal D}$ denote the number of resolutions configured for dataset $\mathcal D$. Resolution $\ell\in\{1,\ldots,L\}$ has region capacity $K_\ell=2^\ell$ and, when active for $G$, contains $k_\ell(G)$ valid graph regions. HP-JEPA represents each valid region using a patch-content token and a structural query, and performs latent prediction independently at each resolution using an online token encoder, an exponential-moving-average (EMA) target token encoder, and a predictor. After pretraining, the frozen target branch produces one graph representation per resolution, and these representations are combined through concatenation or task-specific resolution weighting. Accordingly, HP-JEPA consists of three components: i) multi-resolution graph tokenization; ii) resolution-wise joint-embedding prediction; and iii) multi-resolution readout with task-specific resolution weighting.

\subsection{Multi-Resolution Graph Tokenization}
\label{sec:rh-tokenization}

As illustrated in Figure~\ref{fig:method-overview}(a), HP-JEPA constructs graph tokens at $L$ configured resolutions. A resolution is active for graph $G$ when $K_\ell\leq2n$; once $K_\ell>2n$, that resolution and all finer resolutions are inactive. At an active resolution $\ell$, we set $k_\ell(G):=\min\{K_\ell,n\}$ and partition $V$ into the nonempty core regions $\mathcal P^{(\ell)}(G):=\{P_i^{(\ell)}\}_{i=1}^{k_\ell(G)}$. These regions satisfy $\bigcup_{i=1}^{k_\ell(G)}P_i^{(\ell)}=V$ and $P_i^{(\ell)}\cap P_j^{(\ell)}=\varnothing$ for $i\neq j$. For an inactive resolution, we set $k_\ell(G)=0$. When $K_\ell>n$ at an active resolution, the remaining configured token slots are padded and excluded by a padding mask. Each active resolution is constructed independently from the full input graph, and no nesting or parent--child relation is assumed between $\mathcal P^{(\ell)}(G)$ and $\mathcal P^{(r)}(G)$ for $\ell\neq r$. The resolution capacities and partitioning procedure used in our experiments are specified in Appendix~\ref{sec:experimental-setup}.

For each core region $P_i^{(\ell)}$, let $\widetilde P_i^{(\ell)}\supseteq P_i^{(\ell)}$ denote the node support of its associated graph patch, and let $\rho_i^{(\ell)}$ denote a structural descriptor of the patch in the original graph. A patch-token encoder $g_\eta$ produces the content token $z_i^{(\ell)}:=g_\eta(G[\widetilde P_i^{(\ell)}])\in\mathbb R^d$, while a structural-query encoder $e_\xi$ produces $p_i^{(\ell)}:=e_\xi(\rho_i^{(\ell)})\in\mathbb R^d$. The content token summarizes the attributes and internal structure of the patch, whereas the structural query identifies its structural context. Both encoders are shared across all configured resolutions.

The context and downstream readout paths retain both patch content and structural information through $x_i^{\mathrm{ctx},(\ell)}=x_i^{\mathrm{full},(\ell)}:=z_i^{(\ell)}+p_i^{(\ell)}$, whereas the target path receives only the patch-content input $x_i^{\mathrm{tgt},(\ell)}:=z_i^{(\ell)}$. Consequently, the EMA target token encoder cannot directly observe the target structural query supplied to the predictor. At each resolution, the token encoder operates only on the corresponding valid token inputs, and padded slots are excluded by the appropriate mask. Tokens from different resolutions are not placed in a common sequence during either joint-embedding pretraining or downstream readout. Further implementation details of patch construction, structural descriptors, token encoders, and masking are provided in Appendix~\ref{sec:experimental-setup}.
\vspace{-6pt}
\subsection{Resolution-Wise Joint-Embedding Prediction}
\label{sec:rh-prediction}

As illustrated in Figure~\ref{fig:method-overview}(b), HP-JEPA performs joint-embedding prediction separately at every eligible resolution. For graph $G$ and resolution $\ell$ satisfying $k_\ell(G)\geq2$, let $\mathcal I_\ell(G):=\{1,\ldots,k_\ell(G)\}$ denote the valid region-index set. One context index $c_\ell(G)$ is sampled uniformly from $\mathcal I_\ell(G)$, and a nonempty target-index set $\mathcal T_\ell(G)\subseteq\mathcal I_\ell(G)\setminus\{c_\ell(G)\}$ is sampled without replacement. We denote the number of selected target tokens by $M_\ell(G):=|\mathcal T_\ell(G)|$. The target-set size and sampling rule are specified in Appendix~\ref{sec:experimental-setup}. For readability, we subsequently write $c_\ell=c_\ell(G)$ and $\mathcal T_\ell=\mathcal T_\ell(G)$ when the graph is clear from context.

The online token encoder produces the context state $s_{c_\ell}^{(\ell)}:=f_\theta(x_{c_\ell}^{\mathrm{ctx},(\ell)})\in\mathbb R^d$. For the target branch, define the target-input sequence $X_{\mathcal T}^{(\ell)}:=[x_t^{\mathrm{tgt},(\ell)}]_{t\in\mathcal T_\ell}$ and its encoded sequence $U_{\mathcal T}^{(\ell)}:=f_{\bar\theta}(\operatorname{sg}(X_{\mathcal T}^{(\ell)}))$, where $\operatorname{sg}(\cdot)$ denotes stop-gradient. We use $u_t^{(\ell)}\in\mathbb R^d$ to denote the output in $U_{\mathcal T}^{(\ell)}$ associated with target index $t$. Only $f_{\bar\theta}$ is maintained as an EMA copy of $f_\theta$. The target patch-content inputs are detached before entering $f_{\bar\theta}$, whereas the target structural queries are supplied only to the online predictor and never enter the EMA target token encoder.

For each $t\in\mathcal T_\ell$, the predictor combines the context state with the corresponding target structural query and produces $\widehat y_t^{(\ell)}:=q_\psi(s_{c_\ell}^{(\ell)}+p_t^{(\ell)})\in\mathbb R^2$. Let $\mathbf 1_d\in\mathbb R^d$ denote the all-ones vector, define $m(u):=d^{-1}\mathbf 1_d^\top u$, and define the coordinate map $\Phi(u):=[\cosh(m(u)),\sinh(m(u))]^\top$. The scalar target statistic is therefore $m_t^{(\ell)}:=m(u_t^{(\ell)})$, and the corresponding target coordinate is $y_{t,\mathrm{tgt}}^{(\ell)}:=\Phi(u_t^{(\ell)})\in\mathbb R^2$. With $J_{\mathrm L}:=\operatorname{diag}(-1,1)$, the target satisfies $(y_{t,\mathrm{tgt}}^{(\ell)})^\top J_{\mathrm L}y_{t,\mathrm{tgt}}^{(\ell)}=-1$ and $(y_{t,\mathrm{tgt}}^{(\ell)})_1>0$, and therefore lies on the positive branch of the one-dimensional Lorentz hyperboloid. The predictor output is trained to approximate this target in ambient coordinates but is not itself explicitly constrained to lie on the hyperboloid.

Let $\ell_{G,\ell,t}$ denote the coordinate-wise $\operatorname{SmoothL1}_{\beta}$ penalty between $\widehat y_t^{(\ell)}$ and $\operatorname{sg}(y_{t,\mathrm{tgt}}^{(\ell)})$, averaged over the two output coordinates. The resolution-specific graph loss is $\mathcal L_\ell(G):=M_\ell(G)^{-1}\sum_{t\in\mathcal T_\ell(G)}\ell_{G,\ell,t}$. Thus, the objective predicts a two-coordinate function of the feature mean of each EMA target-token state rather than reconstructing its complete $d$-dimensional representation. The value of $\beta$ is reported in Appendix~\ref{sec:experimental-setup}.

For a minibatch $\mathcal B$, let $\mathcal B_\ell:=\{G\in\mathcal B:k_\ell(G)\geq2\}$ denote the graphs eligible at resolution $\ell$. When $\mathcal B_\ell\neq\varnothing$, define the target-count normalizer $Z_\ell(\mathcal B):=\sum_{G\in\mathcal B_\ell}M_\ell(G)$ and the graph weight $w_{G,\ell}:=M_\ell(G)/Z_\ell(\mathcal B)$. The minibatch objective is then $\mathcal L_\ell(\mathcal B):=\sum_{G\in\mathcal B_\ell}w_{G,\ell}\mathcal L_\ell(G)$. This reduction averages uniformly over all selected target tokens in the minibatch, so a graph with more selected targets contributes proportionally more to the update. Although the context index is excluded from $\mathcal T_\ell(G)$, context and target patches may share nodes when the patch-construction operator expands core regions into overlapping supports. The held-out prediction unit is therefore the target token rather than a strictly node-disjoint subgraph.

HP-JEPA processes the configured resolutions sequentially. For each minibatch, the model visits $\ell=1,\ldots,L$, applies one optimizer update using $\mathcal L_\ell(\mathcal B)$ whenever $\mathcal B_\ell\neq\varnothing$, and then updates the EMA target token encoder as $\bar\theta^{[\nu+1]}:=\mu_\nu\bar\theta^{[\nu]}+(1-\mu_\nu)\theta^{[\nu+1]}$, where $\nu$ is the global optimizer-step index and $\mu_\nu\in[0,1)$ is the corresponding EMA momentum. Resolution-specific losses are not summed before a common optimizer update, and no cross-resolution, adjacent-resolution, or parent--child consistency term is included in the pretraining objective. The optimizer and EMA schedules are specified in Appendix~\ref{sec:experimental-setup}.

Appendix~\ref{app:proof-coordinate-geometry} shows that, with respect to Lorentz geodesic distance, $\Phi$ is an isometric encoding of the feature-mean projection $m(u)=d^{-1}\mathbf 1_d^\top u$. It therefore preserves differences in the supervised feature-mean component while remaining invariant to components in $\mathbf 1_d^\perp:=\{v\in\mathbb R^d:\mathbf 1_d^\top v=0\}$. The same analysis shows that the population optimum of coordinate-wise Smooth L1 is a conditional Huber location of the target coordinates given the context state and target structural query, rather than necessarily their conditional mean.
\vspace{-4pt}
\subsection{Multi-Resolution Readout and Task-Specific Resolution Weighting}
\label{sec:rh-readout-weighting}

As illustrated in Figure~\ref{fig:method-overview}(c), after self-supervised pretraining, the context--target sampling procedure and predictor $q_\psi$ are discarded, while $g_\eta$, $e_\xi$, and $f_{\bar\theta}$ are retained and frozen. For each active resolution $\ell$, define the full-token sequence $X_G^{\mathrm{full},(\ell)}:=[x_i^{\mathrm{full},(\ell)}]_{i=1}^{k_\ell(G)}$ and encode it independently as $R_G^{(\ell)}:=f_{\bar\theta}(X_G^{\mathrm{full},(\ell)})$. We use $r_{G,i}^{(\ell)}\in\mathbb R^d$ to denote the output in $R_G^{(\ell)}$ associated with region $i$. Padded positions are masked and excluded. Because each resolution is processed separately, no cross-resolution token interaction is introduced by the frozen encoder; different resolutions are combined only after resolution-wise pooling.

For an active resolution, the valid token representations are mean-pooled as $h_G^{(\ell)}:=k_\ell(G)^{-1}\sum_{i=1}^{k_\ell(G)}r_{G,i}^{(\ell)}\in\mathbb R^d$. For an inactive resolution with $k_\ell(G)=0$, we set $h_G^{(\ell)}:=\mathbf 0_d$, where $\mathbf 0_d\in\mathbb R^d$ is the zero vector. One downstream representation concatenates all configured resolution slots as $h_G^{\mathrm{concat}}:=h_G^{(1)}\Vert\cdots\Vert h_G^{(L)}\in\mathbb R^{Ld}$, where $\Vert$ denotes vector concatenation. This representation preserves the output associated with every active resolution while maintaining a fixed dimension.

Alternatively, HP-JEPA learns a task-specific mixture of the resolution representations. Let $b^{\mathrm{task}}\in\mathbb R^L$ be the resolution-logit vector and define the raw weights by $\widetilde\omega^{\mathrm{task}}:=\operatorname{softmax}(b^{\mathrm{task}})$. The uniformly smoothed weight for resolution $\ell$ is $\omega_\ell^{\mathrm{task}}:=(1-\lambda_{\mathrm{unif}})\widetilde\omega_\ell^{\mathrm{task}}+\lambda_{\mathrm{unif}}/L$, where $\lambda_{\mathrm{unif}}\in[0,1]$. These weights satisfy $\omega_\ell^{\mathrm{task}}>0$ and $\sum_{\ell=1}^{L}\omega_\ell^{\mathrm{task}}=1$. The resulting task-specific representation is $h_G^{\mathrm{task}}:=\sum_{\ell=1}^{L}\omega_\ell^{\mathrm{task}}h_G^{(\ell)}\in\mathbb R^d$.

The weight vector $\omega^{\mathrm{task}}$ is shared by all graphs within the same downstream task and experimental split, so the weighting is task-specific rather than graph-specific. The weights are defined over all $L$ configured resolution slots and are not renormalized over the graph-specific active subset. Consequently, $\lambda_{\mathrm{unif}}=1$ gives the uniform representation $h_G^{\mathrm{task}}:=L^{-1}\sum_{\ell=1}^{L}h_G^{(\ell)}$, where inactive slots contribute $\mathbf 0_d$, whereas $\lambda_{\mathrm{unif}}=0$ gives the unsmoothed learned mixture.

The concatenated representation $h_G^{\mathrm{concat}}$ is obtained without downstream labels. For $\lambda_{\mathrm{unif}}<1$, the task-specific representation uses downstream labels to estimate the relative contribution of each resolution after self-supervised pretraining, whereas $\lambda_{\mathrm{unif}}=1$ yields label-independent uniform averaging. The patch modules and EMA target token encoder remain frozen throughout downstream training. The procedure used to optimize $b^{\mathrm{task}}$, select its checkpoint, apply uniform smoothing, and fit the final downstream predictor is specified in Table~\ref{tab:downstream_probe_configuration}.

Conditioned on the frozen resolution-specific representations, Appendix~\ref{app:proof-resolution-generalization} shows that the closure of the unsmoothed softmax-parameterized weighting family
contains every fixed-resolution linear predictor.  Under the stated boundedness assumptions, empirical risk minimization over this family competes with the best fixed resolution selected in hindsight, with an additional estimation term of order $O(\sqrt{\log L/N})$ for $N$ labeled downstream examples. Uniform smoothing adds an approximation term that is at most linear in $\lambda_{\mathrm{unif}}$.\begin{table*}[t]
\centering
\caption{Results on graph classification and regression benchmarks. For reference only, we include the results of F-GIN, a pioneering end-to-end supervised GNN. SSL methods are grouped by pretraining paradigm in the following order: contrastive, generative, self-predictive, and JEPA. Baseline results are taken directly from prior work~\cite{hassani2020contrastive, tan2023s2gae, suresh2021adversarial, skenderi2025graphjepa} and "-" indicates that the corresponding results are not reported. The best SSL result is \underline{underlined}, and the second best is \uwave{wavy underlined}. For a more direct comparison, HP-JEPA results that outperform the Graph-JEPA baseline are highlighted with a \colorbox{red!20}{red cell background}.}
\vspace{-10pt}
\label{tab:main_results}
\small
\setlength{\tabcolsep}{4pt}
\renewcommand{\arraystretch}{1.08}
\resizebox{\textwidth}{!}{%
\begin{tabular}{@{}lcccccccc@{}}
\toprule
Model 
& PROTEINS $\uparrow$ 
& MUTAG $\uparrow$ 
& DD $\uparrow$ 
& REDDIT-B $\uparrow$ 
& REDDIT-M5 $\uparrow$ 
& IMDB-B $\uparrow$ 
& IMDB-M $\uparrow$ 
& ZINC-12K $\downarrow$ \\
\midrule
F-GIN (Supervised)
& $72.39 \pm 2.76$ 
& $90.41 \pm 4.61$ 
& $74.87 \pm 3.56$ 
& $86.79 \pm 2.04$ 
& $53.28 \pm 3.17$ 
& $71.83 \pm 1.93$ 
& $48.46 \pm 2.31$ 
& $0.254 \pm 0.005$ \\
\midrule
InfoGraph 
& $72.57 \pm 0.65$ 
& $87.71 \pm 1.77$ 
& $75.23 \pm 0.39$ 
& $78.79 \pm 2.14$ 
& $51.11 \pm 0.55$ 
& $71.11 \pm 0.88$ 
& $48.66 \pm 0.67$ 
& $0.890 \pm 0.017$ \\

GraphCL 
& $72.86 \pm 1.01$ 
& $88.29 \pm 1.31$ 
& $74.70 \pm 0.70$ 
& $82.63 \pm 0.99$ 
& $53.05 \pm 0.40$ 
& $70.80 \pm 0.77$ 
& $48.49 \pm 0.63$ 
& $0.627 \pm 0.013$ \\

MVGRL 
& -- 
& -- 
& -- 
& $84.5 \pm 0.6$ 
& -- 
& $74.2 \pm 0.7$ 
& $51.2 \pm 0.5$ 
& -- \\

AD-GCL-FIX 
& $73.59 \pm 0.65$ 
& $89.25 \pm 1.45$ 
& $74.49 \pm 0.52$ 
& $85.52 \pm 0.79$ 
& $53.00 \pm 0.82$ 
& $71.57 \pm 1.01$ 
& $49.04 \pm 0.53$ 
& $0.578 \pm 0.012$ \\

AD-GCL-OPT 
& $73.81 \pm 0.46$ 
& $89.70 \pm 1.03$ 
& $75.10 \pm 0.39$ 
& $85.52 \pm 0.79$ 
& $54.93 \pm 0.43$ 
& $72.33 \pm 0.56$ 
& $49.89 \pm 0.66$ 
& $0.544 \pm 0.004$ \\
\midrule
GraphMAE 
& $75.30 \pm 0.39$ 
& $88.19 \pm 1.26$ 
& $74.27 \pm 1.07$ 
& $88.01 \pm 0.19$ 
& $46.06 \pm 3.44$ 
& \uwave{$75.52 \pm 0.66$} 
& \uwave{$51.63 \pm 0.52$} 
& $0.935 \pm 0.034$ \\

S2GAE 
& \uwave{$76.37 \pm 0.43$} 
& $88.26 \pm 0.76$ 
& -- 
& $87.83 \pm 0.27$ 
& -- 
& \underline{$75.76 \pm 0.62$} 
& \underline{$51.79 \pm 0.36$} 
& -- \\
\midrule
BGRL 
& $70.99 \pm 3.86$ 
& $74.99 \pm 8.83$ 
& $71.52 \pm 2.97$ 
& $50 \pm 0$ 
& $20 \pm 0.1$ 
& $0.5 \pm 0$ 
& $0.33 \pm 0$ 
& $1.2 \pm 0.011$ \\

LaGraph 
& $75.2 \pm 0.4$ 
& $90.2 \pm 1.1$ 
& $78.1 \pm 0.4$ 
& $90.4 \pm 0.8$ 
& $56.4 \pm 0.4$ 
& $73.7 \pm 0.9$ 
& -- 
& -- \\
\midrule
Graph-JEPA 
& $75.67 \pm 3.78$ 
& \uwave{$91.25 \pm 5.75$} 
& \uwave{$78.64 \pm 2.35$} 
& \uwave{$91.99 \pm 1.59$} 
& \uwave{$56.73 \pm 1.96$} 
& $73.68 \pm 3.24$ 
& $50.69 \pm 2.91$ 
& \underline{$0.434 \pm 0.014$} \\

\textbf{HP-JEPA} 
& \cellcolor{red!20}\underline{$76.97 \pm 0.50$} 
& \cellcolor{red!20}\underline{$91.70 \pm 5.31$} 
& \cellcolor{red!20}\underline{$79.20 \pm 0.38$} 
& \cellcolor{red!20}\underline{$92.19 \pm 1.56$} 
& \cellcolor{red!20}\underline{$57.58 \pm 1.63$} 
& \cellcolor{red!20}$74.08 \pm 0.40$ 
& $50.51 \pm 1.81$ 
& \uwave{$0.472 \pm 0.010$} \\
\bottomrule
\end{tabular}%
}
\vspace{-10pt}
\end{table*}
\vspace{-8pt}
\section{Experiments}
\label{sec:experiments}

In this section, we conduct experiments to substantiate our claims and demonstrate the effectiveness of HP-JEPA. Specifically, Sec.~\ref{sec:exp_downstream} evaluates downstream classification and regression performance, Sec.~\ref{sec:case_study} presents a representation case study on PROTEINS, and Sec.~\ref{sec:sensitivity_analysis} analyzes graph-size and resolution-bank sensitivity.

\subsection{Downstream Task Evaluation}\label{sec:exp_downstream}

\subsubsection{Setup}
Following established graph self-supervised learning evaluation protocols~\cite{skenderi2025graphjepa, suresh2021adversarial, tan2023s2gae}, we evaluate HP-JEPA on seven graph-classification benchmarks from TUDataset~\cite{Morris+2020_tud}: PROTEINS, MUTAG, DD, REDDIT-BINARY, REDDIT-MULTI-5K, IMDB-BINARY, and IMDB-MULTI. We additionally consider the ZINC-12K molecular graph-regression benchmark following Graph-JEPA~\cite{skenderi2025graphjepa}. Additional dataset statistics, preprocessing details, and dataset-specific configurations are provided in Appendix~\ref{app:dataset-details}.

We compare HP-JEPA with the graph SSL baselines InfoGraph~\cite{sun2019infograph}, GraphCL~\cite{you2020graphcl}, MVGRL~\cite{hassani2020contrastive}, AD-GCL-FIX/OPT~\cite{suresh2021adversarial}, GraphMAE~\cite{hou2022graphmae}, S2GAE~\cite{tan2023s2gae}, BGRL~\cite{thakoor2021large}, LaGraph~\cite{xie2022lagraph}, and Graph-JEPA~\cite{skenderi2025graphjepa}. For graph classification, we perform stratified 10-fold cross-validation under five random seeds; accuracy is first averaged across the ten folds for each seed, after which the final mean and standard deviation are computed across the five seed-level results. For ZINC-12K, we use the standard 10,000/1,000/1,000 train/validation/test split and report the test mean absolute error (MAE) averaged over ten random seeds.

To construct the graph-level representation, all valid tokens at each resolution are passed through the trained EMA target encoder and aggregated via mean pooling, yielding \(L\) resolution-specific feature vectors \(\{h_G^{(\ell)}\}_{\ell=1}^{L}\), where \(h_G^{(\ell)}\in\mathbb{R}^{d}\). As a key innovation of HP-JEPA, these \(L\) feature vectors represent information from different resolutions by construction. These representations are subsequently fused through a learnable resolution-logit vector \(b^{\mathrm{task}}\in\mathbb{R}^{L}\), whose softmax-normalized entries \(\omega_\ell^{\mathrm{task}}\) learn the relative importance of each resolution for the downstream task. The resulting weighted representation \(h_G^{\mathrm{task}}=\sum_{\ell=1}^{L}\omega_\ell^{\mathrm{task}}h_G^{(\ell)}\) is evaluated using a linear model: \(L_2\)-regularized logistic regression for the classification tasks and ridge regression for the ZINC dataset.

\subsubsection{Results and Discussion} The results are presented in Table~\ref{tab:main_results}. \textbf{HP-JEPA outperforms the Graph-JEPA baseline across 6 out of 8 tasks, highlighting the importance of hierarchical partitioning in JEPA-style self-supervised learning for capturing multi-resolution graph structures.} In addition, HP-JEPA is the best-performing model overall, achieving the best performance on 5 out of 8 tasks and the second best on 1 task, while Graph-JEPA achieves the second-best performance on 4 out of 8 tasks and the best performance on 1 task. Both Graph-JEPA and HP-JEPA underperform generative methods (GraphMAE and S2GAE) on the two IMDB datasets. This is potentially because the IMDB datasets are relatively small (e.g. only 13 nodes on average for IMDB-M), with node features consisting only of one-hot degree encodings rather than rich semantic attributes. Under such settings, raw signal reconstruction objectives may provide a stronger training signal by directly modeling the graph structure, whereas predictive objectives such as JEPA have less informative contextual signals to exploit. Nevertheless, HP-JEPA still consistently outperforms the Graph-JEPA baseline on these datasets, demonstrating the benefit of hierarchical partitioning even in this challenging setting. \textbf{Overall, these results indicate that JEPA-style training is effective at learning useful graph representations and that our proposed hierarchical partitioning further enhances representation quality by capturing graph structures across multiple resolutions.}
\begin{figure}[h]
\centering
\includegraphics[width=0.45\textwidth]{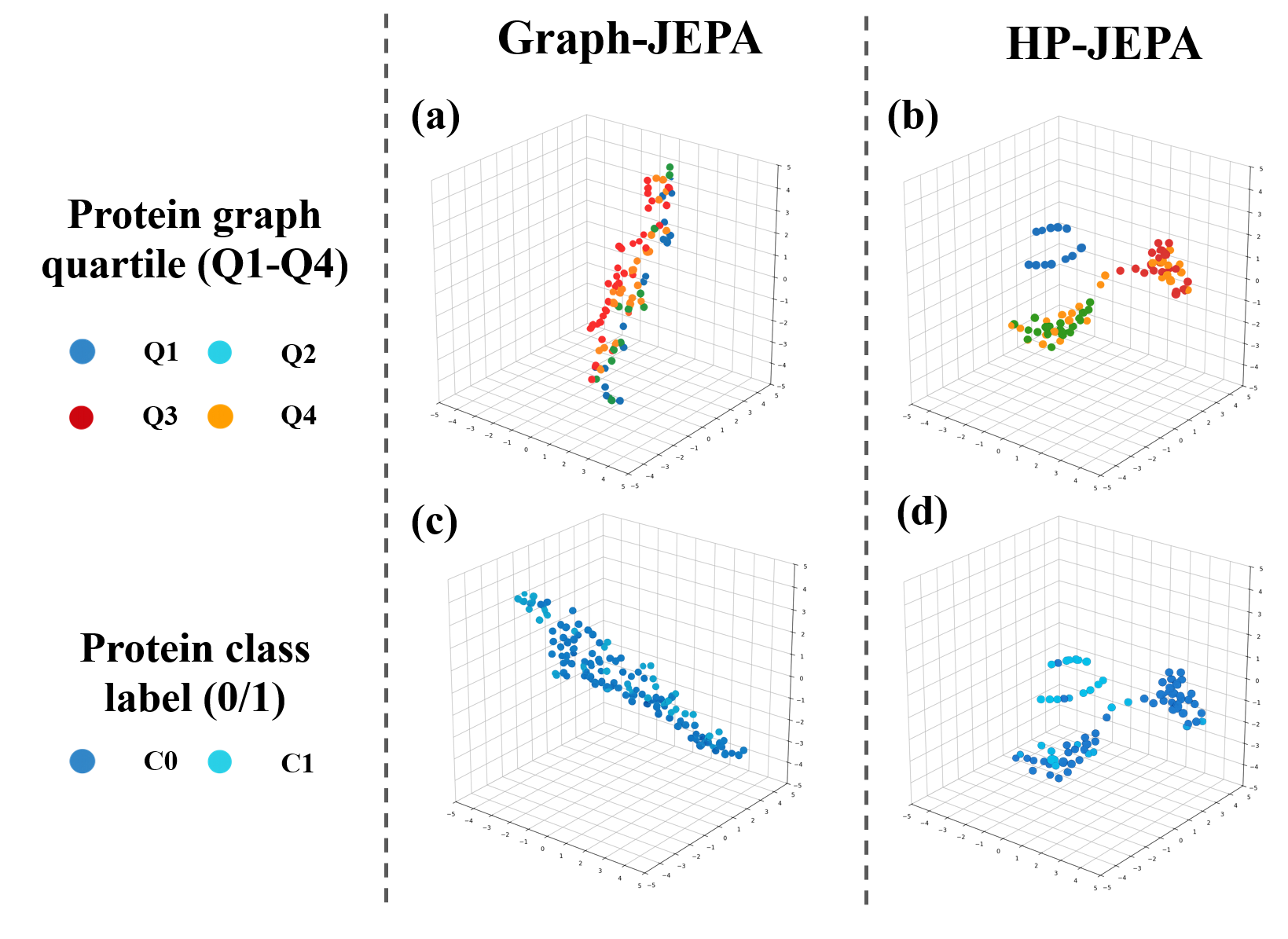}
\vspace{-10pt}
\caption{3D t-SNE visualization of graph-level embeddings on the PROTEINS dataset. Each point represents one protein graph. Columns compare Graph-JEPA and HP-JEPA; the top row colors embeddings by graph-size quartile (Q1–Q4), while the bottom row colors the same embeddings by protein class (C0/C1). Compared with fixed-resolution Graph-JEPA, HP-JEPA produces a more structured representation space, distinctly separating proteins. This alignment indicates that HP-JEPA captures graph-size-dependent structural and functional patterns without using class labels during pretraining. 
}
\vspace{-16pt}
\label{fig:case_study}
\end{figure}
\subsection{Case Study: Learning Multi-Resolution Graph Features}
\label{sec:case_study}

In this section, we examine whether HP-JEPA embeddings retain graph-size-related structural variation and how this variation aligns with downstream
class labels. For the graph classification task on the PROTEINS dataset, both local structural patterns (e.g., functional motifs) and higher-level protein organization (e.g. protein size/graph size and overall topologies) contribute to the prediction, making it a representative benchmark for evaluating multi-resolution representations. 

Specifically, the downstream task on the PROTEINS dataset is binary graph classification: determining whether a protein acts as an enzyme (a biochemical catalyst) or a non-enzyme (e.g., structural scaffold, signaling hormone, or toxin). Biologically, enzymatic catalysis necessitates deep 3D active site clefts, buried hydrophobic cores, and domain-level conformational flexibility, structural prerequisites that inherently require larger molecular scaffolds. In contrast, small micro-proteins lack the volumetric capacity to bury a hydrophobic core, operating instead through solvent-exposed, surface-dominated binding interfaces and localized structural constraints. Consequently, small proteins occupy a fundamentally distinct thermodynamic and functional regime. 

To evaluate whether learned embeddings capture this biophysical transition, we partition protein graphs into node-count quartiles ($Q1-Q4$). As shown in Fig.~\ref{fig:case_study}, standard Graph-JEPA exhibits scale-blindness, collapsing these distinct structural regimes into a single overlapping manifold. In contrast, HP-JEPA achieves topological disentanglement, isolating $Q1$ micro-proteins into a distinct, structured subspace. Critically, Fig.~\ref{fig:case_study} confirms that these $Q1$ embeddings correspond overwhelmingly to non-enzymatic proteins. This demonstrates that HP-JEPA naturally aligns its self-supervised representations with multi-resolution biological function without relying on class labels during pretraining.
\subsection{Sensitivity Analysis}
\label{sec:sensitivity_analysis}
\begin{figure}[t]
\centering
\includegraphics[width=0.48\textwidth ]{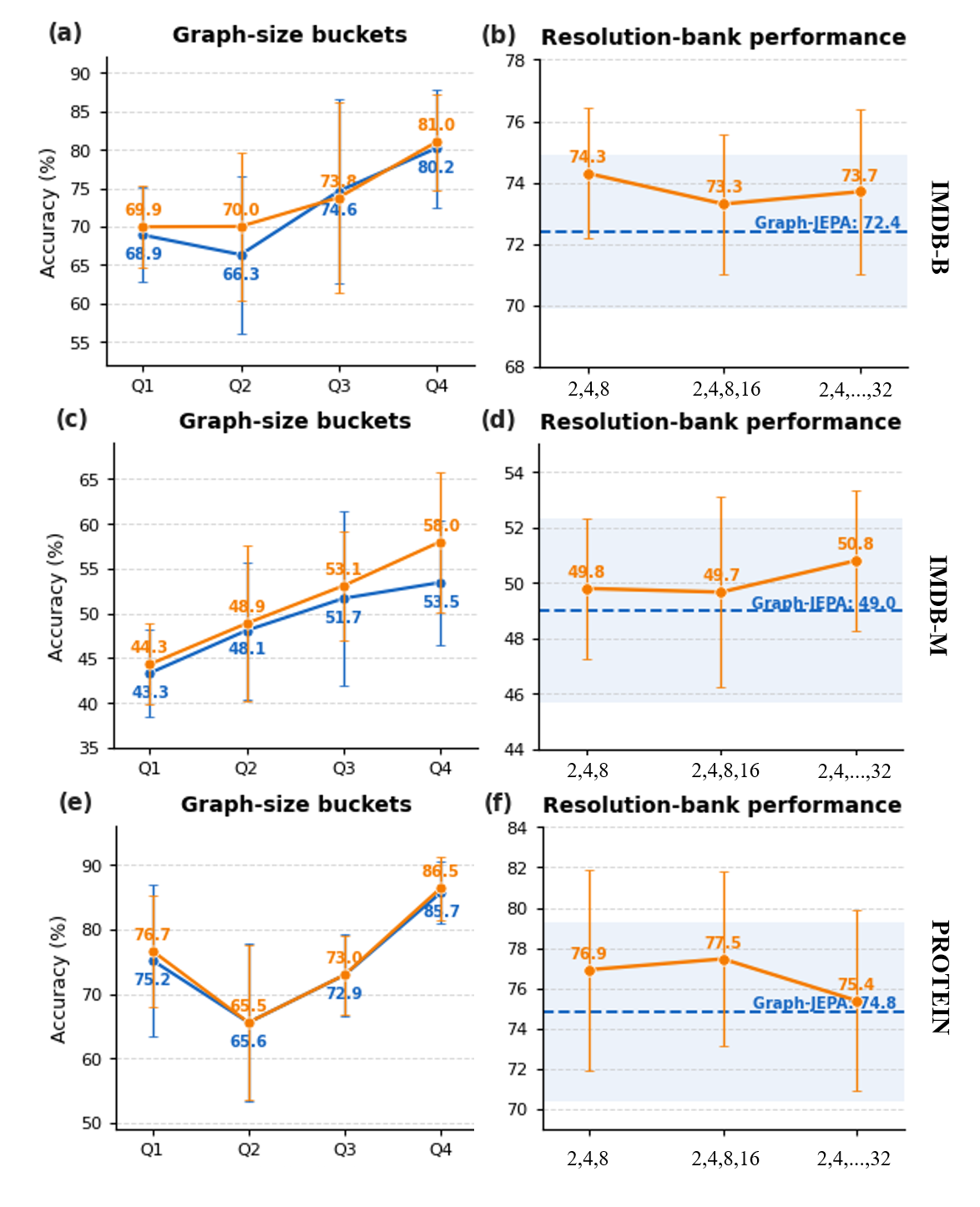}
\vspace{-20pt}
\caption{\textbf{Sensitivity analysis on IMDB-BINARY, IMDB-MULTI, and PROTEINS.}
Panels (a), (c), and (e) report Graph-JEPA and HP-JEPA accuracy across node-count quartiles $Q1$--$Q4$. Panels (b), (d), and (f) compare HP-JEPA using resolution banks $\{2,4,8\}$, $\{2,4,8,16\}$, and the full bank, with Graph-JEPA shown as a dashed reference. Results are the mean and standard deviation over ten cross-validation folds from one pretrained seed.}
\label{fig:sensitivity_analysis}
\vspace{-16pt}
\end{figure}
Figure~\ref{fig:sensitivity_analysis} evaluates the sensitivity of HP-JEPA to graph size and resolution-bank configuration on IMDB-BINARY, IMDB-MULTI, and PROTEINS. Across the graph-size analysis, HP-JEPA achieves higher mean accuracy than Graph-JEPA in most of quartiles. Similarly, HP-JEPA outperforms Graph-JEPA under all nine tested resolution-bank configurations. Although performance varies moderately across resolution banks, the improvement over the fixed-resolution Graph-JEPA remains consistent. By representing each graph at multiple structural granularities, HP-JEPA is less dependent on a single partition resolution that may be overly coarse for large graphs or overly fine for small graphs. Moreover, combining resolution-specific embeddings allows complementary local, regional, and global information to compensate when an individual resolution is suboptimal. In contrast, Graph-JEPA relies on one fixed partition resolution and is therefore more susceptible to mismatches between its predefined granularity and the graph-size distribution. Overall,  within these datasets and the evaluated configurations,
HP-JEPA retains an advantage over Graph-JEPA in most graph-size
quartiles and under all tested resolution banks.


\section{Conclusion}
\label{sec:conclusion}

JEPAs provide a latent predictive approach to graph self-supervised
learning without explicit negative pairs or raw-input reconstruction.
However, existing graph JEPAs rely on a single partition resolution,
restricting the predictive objective to one structural granularity.
We proposed HP-JEPA, which constructs an ordered bank of
coarse-to-fine partition resolutions, performs context--target
prediction independently within each resolution, and integrates the
resulting resolution-specific representations during downstream
readout. HP-JEPA improves over Graph-JEPA on six of the eight
evaluated tasks, including six of seven graph-classification
benchmarks. Size-stratified and resolution-bank analyses further show
that HP-JEPA retains its advantage in most evaluated graph-size
quartiles and across the tested partition banks. These results support
multi-resolution partitioning as a flexible extension of
fixed-resolution graph joint-embedding prediction.


\bibliographystyle{ACM-Reference-Format}
\bibliography{references}


\clearpage
\appendix

\section{Dataset and Training Details}
\label{sec:experimental-setup}
\label{app:training-configuration}
\label{app:dataset-details}

We evaluate HP-JEPA on seven graph-classification benchmarks from TUDataset~\cite{Morris+2020_tud}, including PROTEINS, MUTAG, DD, REDDIT-BINARY, REDDIT-MULTI-5K, IMDB-BINARY, and IMDB-MULTI, together with the ZINC-12K molecular graph-regression benchmark. These benchmarks cover biological, molecular, and social-network domains and exhibit substantial variation in graph size, structural complexity, and label cardinality. Their principal statistics are summarized in Table~\ref{tab:dataset_statistics}, following the dataset statistics reported by Graph-JEPA~\cite{skenderi2025graphjepa}.

\begin{table*}[t]
\centering
\caption{Descriptive statistics of the datasets used in the main experiments. The statistics are adapted from Graph-JEPA~\cite{skenderi2025graphjepa}. The symbol ``--'' indicates that ZINC-12K is a graph-regression benchmark and therefore has no discrete class labels.}
\label{tab:dataset_statistics}
\small
\setlength{\tabcolsep}{14pt}
\renewcommand{\arraystretch}{1.08}
\begin{tabular}{@{}lcccc@{}}
\toprule
Dataset
& Num. Graphs
& Num. Classes
& Avg. Nodes
& Avg. Edges \\
\midrule
PROTEINS
& 1,113
& 2
& 39.06
& 72.82 \\

MUTAG
& 188
& 2
& 17.93
& 19.79 \\

DD
& 1,178
& 2
& 284.32
& 715.66 \\

REDDIT-BINARY
& 2,000
& 2
& 429.63
& 497.75 \\

REDDIT-MULTI-5K
& 4,999
& 5
& 508.52
& 594.87 \\

IMDB-BINARY
& 1,000
& 2
& 19.77
& 96.53 \\

IMDB-MULTI
& 1,500
& 3
& 13.00
& 65.94 \\

ZINC-12K
& 12,000
& --
& 23.20
& 49.80 \\
\bottomrule
\end{tabular}
\end{table*}

For PROTEINS, MUTAG, and DD, we use the node attributes or discrete node labels provided by the corresponding datasets. MUTAG additionally contains categorical edge labels representing chemical bond types.  For ZINC-12K, categorical atom types and bond types are used as node and edge inputs, respectively. Graph-level labels are used only for downstream evaluation and are never provided to HP-JEPA during self-supervised pretraining.

For each graph-classification dataset, we perform stratified 10-fold cross-validation under five random seeds. Under each seed, the dataset is divided into 10 approximately class-balanced folds. For each fold, one fold is reserved exclusively for testing, while the remaining folds are used for self-supervised pretraining and downstream classifier fitting. HP-JEPA is pretrained independently for every fold without access to any graph in the corresponding test fold. After pretraining, graph representations are extracted using the frozen EMA target token encoder, and a linear classifier is fitted using only the labeled training representations. The classification accuracy is first averaged across the 10 folds for each seed, after which the mean and standard deviation are calculated across the five seed-level results.

For ZINC-12K, we use the standard fixed split containing 10,000 training graphs, 1,000 validation graphs, and 1,000 test graphs. The training set is used for self-supervised pretraining and downstream regressor fitting, the validation set is used for model and regularization selection, and the test set is reserved for final evaluation. We report the test mean absolute error (MAE) averaged over 10 random seeds, where a lower value indicates better performance.

\subsection{Pretraining Configuration}
\label{app:pretraining-configuration}

Table~\ref{tab:dataset_training_hyperparameters} reports the selected dataset-specific pretraining configurations. The configured region capacities follow the notation $K_\ell$ introduced in Section~\ref{sec:rh-tokenization}, while $M_\ell$ denotes the number of target regions sampled at the corresponding resolution. All configurations use a latent dimension of 512 and four Transformer blocks. The remaining architectural and optimization settings are selected separately for each benchmark.

\begin{table*}[t]
\centering
\caption{Dataset-specific HP-JEPA pretraining configurations. \(K_\ell\) lists the configured region capacities across resolutions, \(M_\ell\) lists the corresponding numbers of selected target regions, \(d_{\mathrm{RW}}\) is the random-walk positional-encoding dimension, \textsc{Drop} is module dropout, \textsc{Tok.} is token dropout, and \textsc{Clip} is the maximum gradient norm. All configurations use latent dimension 512 and four Transformer blocks.}
\label{tab:dataset_training_hyperparameters}
\scriptsize
\setlength{\tabcolsep}{2.5pt}
\renewcommand{\arraystretch}{1.08}
\resizebox{\textwidth}{!}{%
\begin{tabular}{@{}lccccccccccccc@{}}
\toprule
Dataset
& \(K_\ell\)
& \(M_\ell\)
& GNN Layers
& Heads
& \(d_{\mathrm{RW}}\)
& LR
& WD
& BS
& Epochs
& Scheduler
& Drop.
& Tok.
& Clip \\
\midrule
PROTEINS
& \(2,4,8,16\)
& \(1,2,3,4\)
& 2
& 8
& 15
& \(2.5{\times}10^{-4}\)
& \(1{\times}10^{-5}\)
& 128
& 20
& cosine
& 0
& 0.10
& 2.0 \\

MUTAG
& \(2,4,8,16,32\)
& \(1,2,3,4,4\)
& 2
& 4
& 15
& \(1.5{\times}10^{-4}\)
& 0
& 64
& 30
& cosine
& 0
& 0.05
& 1.0 \\

DD
& \(2,4,8,16,32\)
& \(1,2,3,4,4\)
& 3
& 8
& 30
& \(2{\times}10^{-4}\)
& 0
& 32
& 30
& cosine
& 0
& 0.05
& 1.0 \\

REDDIT-BINARY
& \(2,4,8,16,32,64,128\)
& \(1,2,3,4,4,4,4\)
& 2
& 8
& 40
& \(2{\times}10^{-5}\)
& 0
& 32
& 40
& cosine
& 0
& 0
& 0.5 \\

REDDIT-MULTI-5K
& \(2,4,8,16,32,64,128\)
& \(1,2,3,4,4,4,4\)
& 2
& 8
& 40
& \(2{\times}10^{-5}\)
& 0
& 32
& 40
& cosine
& 0
& 0
& 0.5 \\

IMDB-BINARY
& \(2,4,8\)
& \(1,2,3\)
& 2
& 8
& 15
& \(1{\times}10^{-5}\)
& \(1{\times}10^{-5}\)
& 16
& 12
& constant
& 0.05
& 0
& 2.0 \\

IMDB-MULTI
& \(2,4,8,16,32\)
& \(1,2,3,4,4\)
& 2
& 8
& 15
& \(1{\times}10^{-5}\)
& \(1{\times}10^{-5}\)
& 16
& 12
& cosine
& 0
& 0
& 1.0 \\

ZINC-12K
& \(2,4,8,16,32\)
& \(1,2,3,4,4\)
& 2
& 8
& 20
& \(2{\times}10^{-5}\)
& 0
& 32
& 40
& cosine
& 0
& 0
& 0.5 \\
\bottomrule
\end{tabular}%
}
\end{table*}

Following self-supervised pretraining, graph representations are extracted using the frozen patch-token encoder, structural-query encoder, and EMA target token encoder. At each active resolution $\ell$, valid token representations are mean-pooled to obtain $h_G^{(\ell)}$. The resolution-specific representations are then combined using the dataset-specific readout reported in Table~\ref{tab:downstream_probe_configuration}. For a uniformly smoothed task-specific readout, the coefficient $\lambda_{\mathrm{unif}}$ follows the definition in Section~\ref{sec:rh-readout-weighting}: $\lambda_{\mathrm{unif}}=1$ gives uniform averaging, whereas $\lambda_{\mathrm{unif}}=0.75$ gives a mixture containing $0.75$ uniform weight and $0.25$ learned weight. A linear classifier or regressor is fitted using only the labeled training representations, and its regularization coefficient is selected without using the corresponding test data.

\begin{table}[t]
\centering
\caption{Dataset-specific downstream readout and linear-probe settings. The readout column reports the mixture of uniform and learned resolution weights, and \(\alpha_{\mathrm{probe}}\) denotes the recorded regularization parameter of the downstream linear probe.}
\label{tab:downstream_probe_configuration}
\small
\setlength{\tabcolsep}{3.5pt}
\renewcommand{\arraystretch}{1.08}
\resizebox{\columnwidth}{!}{%
\begin{tabular}{@{}lccc@{}}
\toprule
Dataset
& Resolution Readout
& \(\alpha_{\mathrm{probe}}\)
& Evaluation \\
\midrule
PROTEINS
& \(0.75\) uniform \(+\;0.25\) learned
& \(10^{-2}\)
& \(5\times10\)-fold accuracy \\

MUTAG
& Uniform mean
& \(10^{-1}\)
& \(5\times10\)-fold accuracy \\

DD
& \(0.75\) uniform \(+\;0.25\) learned
& \(10^{-2}\)
& \(5\times10\)-fold accuracy \\

REDDIT-BINARY
& \(0.75\) uniform \(+\;0.25\) learned
& \(10^{-2}\)
& \(5\times10\)-fold accuracy \\

REDDIT-MULTI-5K
& \(0.75\) uniform \(+\;0.25\) learned
& \(10^{-2}\)
& \(5\times10\)-fold accuracy \\

IMDB-BINARY
& \(0.75\) uniform \(+\;0.25\) learned
& \(10^{-2}\)
& \(5\times10\)-fold accuracy \\

IMDB-MULTI
& Uniform mean
& \(10^{-3}\)
& \(5\times10\)-fold accuracy \\

ZINC-12K
& \(0.75\) uniform \(+\;0.25\) learned
& \(10^{-2}\)
& 10-seed MAE \\
\bottomrule
\end{tabular}%
}
\end{table}

\paragraph{Reproducibility and leakage prevention.} Dataset labels are used only when constructing stratified classification splits, estimating task-specific resolution weights, and fitting the downstream linear probe. They are not provided to HP-JEPA during self-supervised pretraining. For every classification fold, the held-out fold is excluded from self-supervised optimization, early stopping, resolution-weight estimation, and downstream probe fitting. For ZINC-12K, the test set is likewise excluded from pretraining, early stopping, readout selection, and probe-regularization selection.

\section{Geometry and Population Characterization of the Latent Target}
\label{app:RP-theory}
\label{app:proof-coordinate-geometry}

This section characterizes the geometric information retained by the latent-coordinate target and the population quantity recovered by the coordinate-wise Smooth L1 objective. The results concern the target construction itself and do not imply that the unconstrained predictor output necessarily lies on the Lorentz hyperboloid.

For $u\in\mathbb R^d$, define
\begin{equation}
m(u)
:=
\frac{1}{d}\mathbf 1_d^\top u,
\qquad
\Phi(u)
:=
\begin{bmatrix}
\cosh(m(u))\\
\sinh(m(u))
\end{bmatrix},
\label{eq:app-coordinate-map}
\end{equation}
where $\mathbf 1_d\in\mathbb R^d$ denotes the all-ones vector. Let
\begin{equation}
J_{\mathrm L}
:=
\begin{bmatrix}
-1 & 0\\
0 & 1
\end{bmatrix}
\label{eq:app-lorentz-matrix}
\end{equation}
and define the positive branch of the one-dimensional Lorentz hyperboloid as
\begin{equation}
\mathbb H_+^1
:=
\left\{
y\in\mathbb R^2:
y^\top J_{\mathrm L}y=-1,\;
y_1>0
\right\}.
\label{eq:app-positive-hyperboloid}
\end{equation}
For $y,y'\in\mathbb H_+^1$, their Lorentz geodesic distance is
\begin{equation}
d_{\mathbb H}(y,y')
:=
\operatorname{arcosh}
\left(
-y^\top J_{\mathrm L}y'
\right).
\label{eq:app-lorentz-distance}
\end{equation}

\begin{proposition}[Geometry and information content of the latent-coordinate map]
\label{prop:latent-coordinate-geometry}
For every $u,v\in\mathbb R^d$, the coordinate map in Equation~\eqref{eq:app-coordinate-map} satisfies
\begin{equation}
\Phi(u)\in\mathbb H_+^1.
\label{eq:app-phi-membership}
\end{equation}
Moreover,
\begin{equation}
d_{\mathbb H}
\left(
\Phi(u),\Phi(v)
\right)
=
|m(u)-m(v)|
=
\frac{1}{d}
\left|
\mathbf 1_d^\top(u-v)
\right|.
\label{eq:app-coordinate-isometry}
\end{equation}
Consequently,
\begin{equation}
d_{\mathbb H}
\left(
\Phi(u),\Phi(v)
\right)
\leq
\frac{1}{\sqrt d}
\|u-v\|_2.
\label{eq:app-coordinate-stability}
\end{equation}
Finally,
\begin{equation}
\Phi(u)=\Phi(v)
\quad\Longleftrightarrow\quad
u-v\in\mathbf 1_d^\perp,
\label{eq:app-coordinate-equivalence}
\end{equation}
where
\begin{equation}
\mathbf 1_d^\perp
:=
\left\{
a\in\mathbb R^d:
\mathbf 1_d^\top a=0
\right\}.
\label{eq:app-orthogonal-subspace}
\end{equation}
Therefore, $\Phi$ preserves exactly the feature-mean projection of the target-token state and is invariant to all latent components orthogonal to $\mathbf 1_d$.
\end{proposition}

\begin{proof}
Let
\begin{equation}
a:=m(u),
\qquad
b:=m(v).
\label{eq:app-proof-ab}
\end{equation}
Using the identity $\cosh^2(a)-\sinh^2(a)=1$, we obtain
\begin{align}
\Phi(u)^\top J_{\mathrm L}\Phi(u)
&=
-\cosh^2(a)+\sinh^2(a)
\nonumber\\
&=
-1.
\label{eq:app-proof-membership}
\end{align}
The first coordinate of $\Phi(u)$ is $\cosh(a)>0$. Hence, $\Phi(u)\in\mathbb H_+^1$.

Next, the Lorentz inner product between $\Phi(u)$ and $\Phi(v)$ satisfies
\begin{align}
-\Phi(u)^\top J_{\mathrm L}\Phi(v)
&=
\cosh(a)\cosh(b)
-
\sinh(a)\sinh(b)
\nonumber\\
&=
\cosh(a-b),
\label{eq:app-proof-inner-product}
\end{align}
where the second equality follows from the hyperbolic difference identity. Therefore,
\begin{align}
d_{\mathbb H}
\left(
\Phi(u),\Phi(v)
\right)
&=
\operatorname{arcosh}
\left(
\cosh(a-b)
\right)
\nonumber\\
&=
|a-b|.
\label{eq:app-proof-distance-ab}
\end{align}
The last equality follows from $\operatorname{arcosh}(\cosh r)=|r|$ for every $r\in\mathbb R$. Substituting the definition of $m$ yields
\begin{align}
d_{\mathbb H}
\left(
\Phi(u),\Phi(v)
\right)
&=
\left|
\frac{1}{d}
\mathbf 1_d^\top u
-
\frac{1}{d}
\mathbf 1_d^\top v
\right|
\nonumber\\
&=
\frac{1}{d}
\left|
\mathbf 1_d^\top(u-v)
\right|.
\label{eq:app-proof-distance-mean}
\end{align}

By the Cauchy--Schwarz inequality,
\begin{align}
\frac{1}{d}
\left|
\mathbf 1_d^\top(u-v)
\right|
&\leq
\frac{1}{d}
\|\mathbf 1_d\|_2
\|u-v\|_2
\nonumber\\
&=
\frac{1}{\sqrt d}
\|u-v\|_2,
\label{eq:app-proof-stability}
\end{align}
because $\|\mathbf 1_d\|_2=\sqrt d$. This proves Equation~\eqref{eq:app-coordinate-stability}.

It remains to characterize when two target-token states have the same coordinate representation. If $\Phi(u)=\Phi(v)$, then their Lorentz distance is zero. Equation~\eqref{eq:app-coordinate-isometry} therefore gives
\begin{equation}
|m(u)-m(v)|=0,
\label{eq:app-proof-equal-means}
\end{equation}
which is equivalent to
\begin{equation}
\mathbf 1_d^\top(u-v)=0.
\label{eq:app-proof-orthogonality}
\end{equation}
Hence, $u-v\in\mathbf 1_d^\perp$. Conversely, if $u-v\in\mathbf 1_d^\perp$, then $m(u)=m(v)$, and the definition of $\Phi$ directly gives $\Phi(u)=\Phi(v)$. This proves Equation~\eqref{eq:app-coordinate-equivalence}.
\end{proof}

We next characterize the population target of coordinate-wise Smooth L1 prediction. For a scalar residual $r\in\mathbb R$ and threshold $\beta>0$, define
\begin{equation}
\varphi_\beta(r)
:=
\begin{cases}
\dfrac{r^2}{2\beta},
&
|r|\leq\beta,
\\[6pt]
|r|-\dfrac{\beta}{2},
&
|r|>\beta.
\end{cases}
\label{eq:app-smooth-l1}
\end{equation}
Let
\begin{equation}
\mathsf X
:=
s_{c_\ell}^{(\ell)}
+
p_t^{(\ell)}
\label{eq:app-predictor-input}
\end{equation}
denote the predictor input and let
\begin{equation}
\mathsf Y
:=
\Phi
\left(
u_t^{(\ell)}
\right)
\in\mathbb R^2
\label{eq:app-random-target}
\end{equation}
denote the corresponding target coordinate. For coordinate $j\in\{1,2\}$ and a fixed input value $x$, define the conditional risk
\begin{equation}
Q_{j,x}(a)
:=
\mathbb E
\left[
\varphi_\beta
\left(
a-\mathsf Y_j
\right)
\,\middle|\,
\mathsf X=x
\right].
\label{eq:app-conditional-huber-risk}
\end{equation}

\begin{proposition}[Population target of coordinate-wise Smooth L1]
\label{prop:conditional-huber-target}
For $\mathbb P_{\mathsf X}$-almost every $x$, suppose that
\begin{equation}
\mathbb E
\left[
|\mathsf Y_j|
\,\middle|\,
\mathsf X=x
\right]
<\infty.
\label{eq:app-huber-moment}
\end{equation}
Then $Q_{j,x}$ is convex, coercive, and admits at least one minimizer. A scalar $a_{j,x}^\star$ minimizes $Q_{j,x}$ if and only if
\begin{equation}
\mathbb E
\left[
\operatorname{clip}
\left(
a_{j,x}^\star-\mathsf Y_j,
-\beta,
\beta
\right)
\,\middle|\,
\mathsf X=x
\right]
=
0.
\label{eq:app-huber-optimality}
\end{equation}
Thus, the population-optimal coordinate-wise Smooth L1 predictor returns a conditional Huber location functional of the target coordinate.

In particular, let
\begin{equation}
\mu_j(x)
:=
\mathbb E
\left[
\mathsf Y_j
\,\middle|\,
\mathsf X=x
\right].
\label{eq:app-conditional-mean}
\end{equation}
If
\begin{equation}
\Pr
\left(
|\mathsf Y_j-\mu_j(x)|
\leq
\beta
\,\middle|\,
\mathsf X=x
\right)
=
1,
\label{eq:app-quadratic-region-condition}
\end{equation}
then $\mu_j(x)$ is a minimizer of $Q_{j,x}$.
\end{proposition}

\begin{proof}
The scalar Smooth L1 function is convex and continuously differentiable. Its derivative is
\begin{equation}
\varphi_\beta'(r)
=
\frac{1}{\beta}
\operatorname{clip}
\left(
r,-\beta,\beta
\right).
\label{eq:app-smooth-l1-derivative}
\end{equation}
In particular,
\begin{equation}
|\varphi_\beta'(r)|
\leq
1
\label{eq:app-smooth-l1-bounded-derivative}
\end{equation}
for every $r\in\mathbb R$. The bounded derivative permits differentiation under the conditional expectation, yielding
\begin{align}
Q_{j,x}'(a)
&=
\mathbb E
\left[
\varphi_\beta'
\left(
a-\mathsf Y_j
\right)
\,\middle|\,
\mathsf X=x
\right]
\nonumber\\
&=
\frac{1}{\beta}
\mathbb E
\left[
\operatorname{clip}
\left(
a-\mathsf Y_j,
-\beta,
\beta
\right)
\,\middle|\,
\mathsf X=x
\right].
\label{eq:app-huber-risk-derivative}
\end{align}

Because $\varphi_\beta$ is convex, its conditional expectation $Q_{j,x}$ is also convex. Moreover, for every $r\in\mathbb R$,
\begin{equation}
\varphi_\beta(r)
\geq
|r|-\frac{\beta}{2}.
\label{eq:app-huber-lower-bound}
\end{equation}
Therefore,
\begin{align}
Q_{j,x}(a)
&\geq
\mathbb E
\left[
|a-\mathsf Y_j|
\,\middle|\,
\mathsf X=x
\right]
-
\frac{\beta}{2}
\nonumber\\
&\geq
|a|
-
\mathbb E
\left[
|\mathsf Y_j|
\,\middle|\,
\mathsf X=x
\right]
-
\frac{\beta}{2}.
\label{eq:app-huber-coercivity}
\end{align}
Under Equation~\eqref{eq:app-huber-moment}, the right-hand side approaches infinity as $|a|\rightarrow\infty$. Hence, $Q_{j,x}$ is coercive and admits at least one minimizer.

Since $Q_{j,x}$ is convex and differentiable, a scalar $a_{j,x}^\star$ is a minimizer if and only if
\begin{equation}
Q_{j,x}'(a_{j,x}^\star)=0.
\label{eq:app-huber-zero-gradient}
\end{equation}
Substituting Equation~\eqref{eq:app-huber-risk-derivative} proves Equation~\eqref{eq:app-huber-optimality}.

Now suppose that Equation~\eqref{eq:app-quadratic-region-condition} holds. At $a=\mu_j(x)$, clipping is inactive almost surely conditional on $\mathsf X=x$, and therefore
\begin{equation}
\operatorname{clip}
\left(
\mu_j(x)-\mathsf Y_j,
-\beta,
\beta
\right)
=
\mu_j(x)-\mathsf Y_j.
\label{eq:app-huber-mean-no-clipping}
\end{equation}
It follows that
\begin{align}
\mathbb E
\left[
\operatorname{clip}
\left(
\mu_j(x)-\mathsf Y_j,
-\beta,
\beta
\right)
\,\middle|\,
\mathsf X=x
\right]
&=
\mu_j(x)
-
\mathbb E
\left[
\mathsf Y_j
\,\middle|\,
\mathsf X=x
\right]
\nonumber\\
&=
0.
\label{eq:app-huber-mean-optimality}
\end{align}
Equation~\eqref{eq:app-huber-optimality} then shows that $\mu_j(x)$ is a minimizer. Since the vector-valued Smooth L1 loss used by HP-JEPA averages the two scalar coordinate losses, its population minimization separates coordinate-wise.
\end{proof}

\section{Generalization of Task-Specific Resolution Weighting}
\label{app:proof-resolution-generalization}

This section analyzes the task-specific resolution-weighting family after self-supervised pretraining. We condition on the pretrained encoder and treat the resolution-specific representations as fixed feature maps. The analysis concerns the downstream hypothesis class and does not establish convergence of the nonlinear pretraining procedure or of a particular nonconvex optimization algorithm.

Let
\begin{equation}
S
=
\left\{
(G_i,Y_i)
\right\}_{i=1}^{N}
\label{eq:app-downstream-sample}
\end{equation}
be $N$ independent labeled downstream examples drawn from a distribution $\mathbb P$. For each resolution $\ell\in\{1,\ldots,L\}$, define the fixed-resolution linear predictor class
\begin{equation}
\mathcal F_\ell
:=
\left\{
f_{\ell,w}(G)
=
\left\langle
w,h_G^{(\ell)}
\right\rangle
:
\|w\|_2\leq B
\right\}.
\label{eq:app-fixed-resolution-class}
\end{equation}
Let
\begin{equation}
\Delta_L
:=
\left\{
\alpha\in\mathbb R_+^L:
\sum_{\ell=1}^{L}\alpha_\ell=1
\right\}
\label{eq:app-closed-simplex}
\end{equation}
denote the closed probability simplex. The resolution-weighted class is
\begin{equation}
\mathcal F_{\mathrm{mix}}
:=
\left\{
f_{w,\alpha}(G)
=
\left\langle
w,
\sum_{\ell=1}^{L}
\alpha_\ell h_G^{(\ell)}
\right\rangle
:
\|w\|_2\leq B,\;
\alpha\in\Delta_L
\right\}.
\label{eq:app-resolution-mixture-class}
\end{equation}

\begin{assumption}[Bounded downstream hypothesis class]
\label{ass:app-bounded-downstream}
There exist constants $R_h,B,\kappa,M>0$ such that the following conditions hold:
\begin{align}
\|h_G^{(\ell)}\|_2
&\leq
R_h
&&
\text{for every $G$ and $\ell$},
\label{eq:app-bounded-feature}
\\
\|w\|_2
&\leq
B,
\label{eq:app-bounded-head}
\\
0
\leq
\ell_{\mathrm{ds}}(\widehat y,y)
&\leq
M,
\label{eq:app-bounded-loss}
\\
\left|
\ell_{\mathrm{ds}}(a,y)
-
\ell_{\mathrm{ds}}(b,y)
\right|
&\leq
\kappa|a-b|
&&
\text{for every $a,b,y$}.
\label{eq:app-lipschitz-loss}
\end{align}
\end{assumption}

For a predictor $f$, define its population and empirical risks by
\begin{equation}
\mathcal R(f)
:=
\mathbb E_{(G,Y)\sim\mathbb P}
\left[
\ell_{\mathrm{ds}}
\left(
f(G),Y
\right)
\right]
\label{eq:app-population-risk}
\end{equation}
and
\begin{equation}
\widehat{\mathcal R}_S(f)
:=
\frac{1}{N}
\sum_{i=1}^{N}
\ell_{\mathrm{ds}}
\left(
f(G_i),Y_i
\right).
\label{eq:app-empirical-risk}
\end{equation}
Let $\widehat f_{\mathrm{mix}}\in\mathcal F_{\mathrm{mix}}$ be an $\varepsilon_{\mathrm{opt}}$-approximate empirical-risk minimizer satisfying
\begin{equation}
\widehat{\mathcal R}_S
\left(
\widehat f_{\mathrm{mix}}
\right)
\leq
\inf_{f\in\mathcal F_{\mathrm{mix}}}
\widehat{\mathcal R}_S(f)
+
\varepsilon_{\mathrm{opt}}.
\label{eq:app-approximate-erm}
\end{equation}

\begin{theorem}[Resolution-adaptive oracle bound]
\label{thm:resolution-adaptive-oracle}
Under Assumption~\ref{ass:app-bounded-downstream}, there exist universal constants $C_1,C_2>0$ such that, for every $\delta\in(0,1)$, with probability at least $1-\delta$,
\begin{align}
\mathcal R
\left(
\widehat f_{\mathrm{mix}}
\right)
\leq\;&
\min_{\ell\in\{1,\ldots,L\}}
\inf_{f\in\mathcal F_\ell}
\mathcal R(f)
\nonumber\\
&+
C_1\kappa BR_h
\sqrt{
\frac{1+\log L}{N}
}
+
C_2M
\sqrt{
\frac{\log(2/\delta)}{N}
}
+
\varepsilon_{\mathrm{opt}}.
\label{eq:app-resolution-oracle-bound}
\end{align}
Thus, the closed resolution-weighting family competes with the best fixed-resolution linear predictor selected in hindsight, while its additional dependence on the number of configured resolutions is logarithmic in $L$.
\end{theorem}

\begin{proof}
For a scalar function class $\mathcal F$, define its empirical Rademacher complexity on the graph inputs in $S$ as
\begin{equation}
\widehat{\mathfrak R}_S(\mathcal F)
:=
\mathbb E_\sigma
\left[
\sup_{f\in\mathcal F}
\frac{1}{N}
\sum_{i=1}^{N}
\sigma_i f(G_i)
\right],
\label{eq:app-empirical-rademacher}
\end{equation}
where $\sigma_1,\ldots,\sigma_N$ are independent Rademacher random variables.

For a fixed resolution $\ell$, Euclidean norm duality gives
\begin{align}
\widehat{\mathfrak R}_S(\mathcal F_\ell)
&=
\mathbb E_\sigma
\left[
\sup_{\|w\|_2\leq B}
\frac{1}{N}
\sum_{i=1}^{N}
\sigma_i
\left\langle
w,h_{G_i}^{(\ell)}
\right\rangle
\right]
\nonumber\\
&=
\frac{B}{N}
\mathbb E_\sigma
\left\|
\sum_{i=1}^{N}
\sigma_i h_{G_i}^{(\ell)}
\right\|_2.
\label{eq:app-single-rad-duality}
\end{align}
By Jensen's inequality,
\begin{align}
\mathbb E_\sigma
\left\|
\sum_{i=1}^{N}
\sigma_i h_{G_i}^{(\ell)}
\right\|_2
&\leq
\left(
\mathbb E_\sigma
\left\|
\sum_{i=1}^{N}
\sigma_i h_{G_i}^{(\ell)}
\right\|_2^2
\right)^{1/2}.
\label{eq:app-single-rad-jensen}
\end{align}
Expanding the squared norm and using independence and zero mean of the Rademacher variables eliminates the cross terms, so
\begin{align}
\mathbb E_\sigma
\left\|
\sum_{i=1}^{N}
\sigma_i h_{G_i}^{(\ell)}
\right\|_2^2
&=
\sum_{i=1}^{N}
\left\|
h_{G_i}^{(\ell)}
\right\|_2^2
\nonumber\\
&\leq
NR_h^2.
\label{eq:app-single-rad-square}
\end{align}
Consequently,
\begin{equation}
\widehat{\mathfrak R}_S(\mathcal F_\ell)
\leq
\frac{BR_h}{\sqrt N}.
\label{eq:app-single-resolution-rad-bound}
\end{equation}

We next consider the union of the $L$ fixed-resolution classes. Define
\begin{equation}
Z_\ell(\sigma)
:=
\sup_{f\in\mathcal F_\ell}
\frac{1}{N}
\sum_{i=1}^{N}
\sigma_i f(G_i).
\label{eq:app-resolution-rad-variable}
\end{equation}
For every $f\in\mathcal F_\ell$, Assumption~\ref{ass:app-bounded-downstream} gives
\begin{equation}
|f(G_i)|
\leq
\|w\|_2
\|h_{G_i}^{(\ell)}\|_2
\leq
BR_h.
\label{eq:app-bounded-prediction}
\end{equation}
Changing one Rademacher sign changes $Z_\ell$ by at most $2BR_h/N$. The bounded-difference exponential-moment inequality therefore yields
\begin{equation}
\mathbb E_\sigma
\left[
\exp
\left(
\lambda
\left[
Z_\ell(\sigma)
-
\mathbb E_\sigma Z_\ell(\sigma)
\right]
\right)
\right]
\leq
\exp
\left(
\frac{\lambda^2B^2R_h^2}{2N}
\right)
\label{eq:app-resolution-subgaussian}
\end{equation}
for every $\lambda>0$. Applying the log-sum-exp inequality gives
\begin{align}
\mathbb E_\sigma
\left[
\max_{\ell\in\{1,\ldots,L\}}
Z_\ell(\sigma)
\right]
\leq\;&
\max_\ell
\mathbb E_\sigma
\left[
Z_\ell(\sigma)
\right]
+
\frac{\log L}{\lambda}
+
\frac{\lambda B^2R_h^2}{2N}.
\label{eq:app-union-log-sum-exp}
\end{align}
Optimizing the right-hand side over $\lambda$ gives
\begin{align}
\widehat{\mathfrak R}_S
\left(
\bigcup_{\ell=1}^{L}\mathcal F_\ell
\right)
&\leq
\max_\ell
\widehat{\mathfrak R}_S(\mathcal F_\ell)
+
BR_h
\sqrt{
\frac{2\log L}{N}
}
\nonumber\\
&\leq
BR_h
\left(
\frac{1}{\sqrt N}
+
\sqrt{
\frac{2\log L}{N}
}
\right).
\label{eq:app-union-rad-bound}
\end{align}

For every $w$ and $\alpha\in\Delta_L$,
\begin{align}
f_{w,\alpha}(G)
&=
\left\langle
w,
\sum_{\ell=1}^{L}
\alpha_\ell h_G^{(\ell)}
\right\rangle
\nonumber\\
&=
\sum_{\ell=1}^{L}
\alpha_\ell
\left\langle
w,h_G^{(\ell)}
\right\rangle.
\label{eq:app-mixture-convex-combination}
\end{align}
Thus,
\begin{equation}
\mathcal F_{\mathrm{mix}}
\subseteq
\operatorname{conv}
\left(
\bigcup_{\ell=1}^{L}\mathcal F_\ell
\right).
\label{eq:app-mixture-convex-hull}
\end{equation}
The supremum of a linear functional over a convex hull equals its supremum over the original set. It follows that
\begin{align}
\widehat{\mathfrak R}_S
\left(
\mathcal F_{\mathrm{mix}}
\right)
&\leq
\widehat{\mathfrak R}_S
\left(
\operatorname{conv}
\left(
\bigcup_{\ell=1}^{L}\mathcal F_\ell
\right)
\right)
\nonumber\\
&=
\widehat{\mathfrak R}_S
\left(
\bigcup_{\ell=1}^{L}\mathcal F_\ell
\right).
\label{eq:app-mixture-rad-convex-hull}
\end{align}
Combining Equations~\eqref{eq:app-union-rad-bound} and~\eqref{eq:app-mixture-rad-convex-hull}, and taking expectation over the sample, gives
\begin{equation}
\mathfrak R_N
\left(
\mathcal F_{\mathrm{mix}}
\right)
\leq
BR_h
\left(
\frac{1}{\sqrt N}
+
\sqrt{
\frac{2\log L}{N}
}
\right),
\label{eq:app-mixture-expected-rad}
\end{equation}
where $\mathfrak R_N$ denotes the expected Rademacher complexity.

By the contraction principle and the $\kappa$-Lipschitz property of $\ell_{\mathrm{ds}}$, the complexity of the induced loss class is bounded by $\kappa\mathfrak R_N(\mathcal F_{\mathrm{mix}})$. A standard Rademacher uniform-convergence bound for bounded losses therefore implies that there exists a universal constant $C_0>0$ such that, with probability at least $1-\delta$,
\begin{align}
\sup_{f\in\mathcal F_{\mathrm{mix}}}
\left|
\mathcal R(f)
-
\widehat{\mathcal R}_S(f)
\right|
\leq\;&
C_0\kappa BR_h
\left(
\frac{1}{\sqrt N}
+
\sqrt{
\frac{2\log L}{N}
}
\right)
\nonumber\\
&+
C_0M
\sqrt{
\frac{\log(2/\delta)}{N}
}.
\label{eq:app-uniform-convergence}
\end{align}
Denote the right-hand side by $\Gamma_N(\delta)$. On the event in Equation~\eqref{eq:app-uniform-convergence},
\begin{align}
\mathcal R
\left(
\widehat f_{\mathrm{mix}}
\right)
&\leq
\widehat{\mathcal R}_S
\left(
\widehat f_{\mathrm{mix}}
\right)
+
\Gamma_N(\delta)
\nonumber\\
&\leq
\inf_{f\in\mathcal F_{\mathrm{mix}}}
\widehat{\mathcal R}_S(f)
+
\varepsilon_{\mathrm{opt}}
+
\Gamma_N(\delta)
\nonumber\\
&\leq
\inf_{f\in\mathcal F_{\mathrm{mix}}}
\mathcal R(f)
+
\varepsilon_{\mathrm{opt}}
+
2\Gamma_N(\delta).
\label{eq:app-erm-risk-chain}
\end{align}

For every resolution $\ell$, the standard basis vector $e_\ell$ belongs to the closed simplex $\Delta_L$. Choosing $\alpha=e_\ell$ gives
\begin{equation}
f_{w,e_\ell}(G)
=
\left\langle
w,h_G^{(\ell)}
\right\rangle.
\label{eq:app-vertex-fixed-resolution}
\end{equation}
Therefore,
\begin{equation}
\bigcup_{\ell=1}^{L}
\mathcal F_\ell
\subseteq
\mathcal F_{\mathrm{mix}},
\label{eq:app-fixed-in-mixture}
\end{equation}
and hence
\begin{equation}
\inf_{f\in\mathcal F_{\mathrm{mix}}}
\mathcal R(f)
\leq
\min_{\ell\in\{1,\ldots,L\}}
\inf_{f\in\mathcal F_\ell}
\mathcal R(f).
\label{eq:app-oracle-comparator}
\end{equation}
Substituting Equations~\eqref{eq:app-uniform-convergence} and~\eqref{eq:app-oracle-comparator} into Equation~\eqref{eq:app-erm-risk-chain}, and absorbing numerical factors into universal constants $C_1$ and $C_2$, proves Equation~\eqref{eq:app-resolution-oracle-bound}.
\end{proof}

\paragraph{Relation to the softmax parameterization.} A finite softmax logit vector produces a weight vector in the interior of $\Delta_L$, whereas a fixed-resolution predictor corresponds to a simplex vertex. Therefore, the exact class-containment statement in Theorem~\ref{thm:resolution-adaptive-oracle} applies to the closure of the softmax-parameterized family. Every simplex vertex can be approached arbitrarily closely by increasing one resolution logit relative to the remaining logits.

We next quantify the effect of the uniform smoothing used in the task-specific readout. Let
\begin{equation}
u_L
:=
\frac{1}{L}\mathbf 1_L
\label{eq:app-uniform-weight-vector}
\end{equation}
and, for $\alpha\in\Delta_L$, define
\begin{equation}
\alpha^{(\lambda)}
:=
(1-\lambda_{\mathrm{unif}})\alpha
+
\lambda_{\mathrm{unif}}u_L,
\qquad
\lambda_{\mathrm{unif}}\in[0,1].
\label{eq:app-smoothed-weight}
\end{equation}

\begin{lemma}[Stability under uniform resolution smoothing]
\label{lem:uniform-smoothing-stability}
For $\alpha\in\Delta_L$, define
\begin{equation}
h_G^\alpha
:=
\sum_{\ell=1}^{L}
\alpha_\ell h_G^{(\ell)}
\label{eq:app-unsmoothed-representation}
\end{equation}
and
\begin{equation}
h_G^{\alpha^{(\lambda)}}
:=
\sum_{\ell=1}^{L}
\alpha_\ell^{(\lambda)}
h_G^{(\ell)}.
\label{eq:app-smoothed-representation}
\end{equation}
Under Assumption~\ref{ass:app-bounded-downstream},
\begin{equation}
\left\|
h_G^{\alpha^{(\lambda)}}
-
h_G^\alpha
\right\|_2
\leq
2\lambda_{\mathrm{unif}}R_h.
\label{eq:app-smoothing-feature-bound}
\end{equation}
Consequently, for every $\|w\|_2\leq B$,
\begin{equation}
\left|
\left\langle
w,h_G^{\alpha^{(\lambda)}}
\right\rangle
-
\left\langle
w,h_G^\alpha
\right\rangle
\right|
\leq
2\lambda_{\mathrm{unif}}BR_h.
\label{eq:app-smoothing-prediction-bound}
\end{equation}
The corresponding population risks satisfy
\begin{equation}
\mathcal R
\left(
f_{w,\alpha^{(\lambda)}}
\right)
\leq
\mathcal R
\left(
f_{w,\alpha}
\right)
+
2\kappa\lambda_{\mathrm{unif}}BR_h.
\label{eq:app-smoothing-risk-bound}
\end{equation}
\end{lemma}

\begin{proof}
Using Equation~\eqref{eq:app-smoothed-weight},
\begin{align}
h_G^{\alpha^{(\lambda)}}
-
h_G^\alpha
&=
\lambda_{\mathrm{unif}}
\left(
\frac{1}{L}
\sum_{\ell=1}^{L}
h_G^{(\ell)}
-
\sum_{\ell=1}^{L}
\alpha_\ell h_G^{(\ell)}
\right).
\label{eq:app-smoothing-difference}
\end{align}
The triangle inequality and Assumption~\ref{ass:app-bounded-downstream} give
\begin{align}
\left\|
h_G^{\alpha^{(\lambda)}}
-
h_G^\alpha
\right\|_2
\leq\;&
\lambda_{\mathrm{unif}}
\left\|
\frac{1}{L}
\sum_{\ell=1}^{L}
h_G^{(\ell)}
\right\|_2
+
\lambda_{\mathrm{unif}}
\left\|
\sum_{\ell=1}^{L}
\alpha_\ell h_G^{(\ell)}
\right\|_2
\nonumber\\
\leq\;&
\lambda_{\mathrm{unif}}
\frac{1}{L}
\sum_{\ell=1}^{L}
\left\|
h_G^{(\ell)}
\right\|_2
+
\lambda_{\mathrm{unif}}
\sum_{\ell=1}^{L}
\alpha_\ell
\left\|
h_G^{(\ell)}
\right\|_2
\nonumber\\
\leq\;&
2\lambda_{\mathrm{unif}}R_h.
\label{eq:app-proof-smoothing-feature}
\end{align}
This proves Equation~\eqref{eq:app-smoothing-feature-bound}.

By the Cauchy--Schwarz inequality,
\begin{align}
\left|
\left\langle
w,
h_G^{\alpha^{(\lambda)}}
-
h_G^\alpha
\right\rangle
\right|
&\leq
\|w\|_2
\left\|
h_G^{\alpha^{(\lambda)}}
-
h_G^\alpha
\right\|_2
\nonumber\\
&\leq
2\lambda_{\mathrm{unif}}BR_h,
\label{eq:app-proof-smoothing-prediction}
\end{align}
which proves Equation~\eqref{eq:app-smoothing-prediction-bound}.

Finally, the $\kappa$-Lipschitz property of the downstream loss implies
\begin{align}
&
\left|
\ell_{\mathrm{ds}}
\left(
f_{w,\alpha^{(\lambda)}}(G),Y
\right)
-
\ell_{\mathrm{ds}}
\left(
f_{w,\alpha}(G),Y
\right)
\right|
\nonumber\\
&\hspace{35mm}
\leq
\kappa
\left|
f_{w,\alpha^{(\lambda)}}(G)
-
f_{w,\alpha}(G)
\right|
\nonumber\\
&\hspace{35mm}
\leq
2\kappa\lambda_{\mathrm{unif}}BR_h.
\label{eq:app-proof-smoothing-loss}
\end{align}
Taking expectation over $(G,Y)\sim\mathbb P$ proves Equation~\eqref{eq:app-smoothing-risk-bound}.
\end{proof}

Define the uniformly smoothed hypothesis class as
\begin{equation}
\mathcal F_{\mathrm{mix}}^{(\lambda)}
:=
\left\{
f_{w,\alpha^{(\lambda)}}:
\|w\|_2\leq B,\;
\alpha\in\Delta_L
\right\}.
\label{eq:app-smoothed-mixture-class}
\end{equation}
Let $\widehat f_{\mathrm{mix}}^{(\lambda)}$ be an $\varepsilon_{\mathrm{opt}}$-approximate empirical-risk minimizer over $\mathcal F_{\mathrm{mix}}^{(\lambda)}$.

\begin{corollary}[Oracle bound under uniform smoothing]
\label{cor:smoothed-resolution-oracle}
Under Assumption~\ref{ass:app-bounded-downstream}, there exist universal constants $C_1,C_2>0$ such that, for every $\delta\in(0,1)$, with probability at least $1-\delta$,
\begin{align}
\mathcal R
\left(
\widehat f_{\mathrm{mix}}^{(\lambda)}
\right)
\leq\;&
\min_{\ell\in\{1,\ldots,L\}}
\inf_{f\in\mathcal F_\ell}
\mathcal R(f)
+
2\kappa\lambda_{\mathrm{unif}}BR_h
\nonumber\\
&+
C_1\kappa BR_h
\sqrt{
\frac{1+\log L}{N}
}
+
C_2M
\sqrt{
\frac{\log(2/\delta)}{N}
}
+
\varepsilon_{\mathrm{opt}}.
\label{eq:app-smoothed-oracle-bound}
\end{align}
Therefore, uniform smoothing introduces an additional approximation term that is at most linear in $\lambda_{\mathrm{unif}}$.
\end{corollary}

\begin{proof}
Since $\alpha^{(\lambda)}\in\Delta_L$ whenever $\alpha\in\Delta_L$,
\begin{equation}
\mathcal F_{\mathrm{mix}}^{(\lambda)}
\subseteq
\mathcal F_{\mathrm{mix}}.
\label{eq:app-smoothed-class-subset}
\end{equation}
Consequently, the Rademacher-complexity and uniform-convergence bounds derived in the proof of Theorem~\ref{thm:resolution-adaptive-oracle} also apply to $\mathcal F_{\mathrm{mix}}^{(\lambda)}$.

Fix a resolution $\ell$ and a vector $w$ satisfying $\|w\|_2\leq B$. Let $e_\ell$ denote the $\ell$-th standard basis vector in $\mathbb R^L$. The corresponding unsmoothed representation is
\begin{equation}
h_G^{e_\ell}
=
h_G^{(\ell)}.
\label{eq:app-fixed-resolution-vertex}
\end{equation}
Applying Lemma~\ref{lem:uniform-smoothing-stability} with $\alpha=e_\ell$ gives
\begin{equation}
\mathcal R
\left(
f_{w,e_\ell^{(\lambda)}}
\right)
\leq
\mathcal R
\left(
f_{\ell,w}
\right)
+
2\kappa\lambda_{\mathrm{unif}}BR_h.
\label{eq:app-smoothed-fixed-comparator}
\end{equation}
Taking the infimum over $\|w\|_2\leq B$ and then the minimum over $\ell$ yields
\begin{align}
\inf_{f\in\mathcal F_{\mathrm{mix}}^{(\lambda)}}
\mathcal R(f)
\leq\;&
\min_{\ell\in\{1,\ldots,L\}}
\inf_{f\in\mathcal F_\ell}
\mathcal R(f)
+
2\kappa\lambda_{\mathrm{unif}}BR_h.
\label{eq:app-smoothed-comparator}
\end{align}
Applying the same approximate empirical-risk-minimization and uniform-convergence argument as in Equation~\eqref{eq:app-erm-risk-chain}, together with Equation~\eqref{eq:app-smoothed-comparator}, proves Equation~\eqref{eq:app-smoothed-oracle-bound}.
\end{proof}

\end{document}